\documentclass[10pt,journal,compsoc]{IEEEtran}

\usepackage{amsmath,amsfonts, amsthm, bm, amssymb}

\def\Tabref#1{Table~\ref{#1}}

\def\Figref#1{Figure~\ref{#1}}

\def\Secref#1{Section~\ref{#1}}

\def\Twosecrefs#1#2{Sections \ref{#1} and \ref{#2}}

\def\eqref#1{equation~(\ref{#1})}
\def\Eqref#1{Equation~(\ref{#1})}

\def\Algref#1{Algorithm~\ref{#1}}

\def\plainref#1{(\ref{#1})}

\def\1{\bm{1}}

\def\rk{{\textnormal{k}}}

\def\vzero{{\bm{0}}}
\def\vone{{\bm{1}}}

\def\vtheta{{\bm{\theta}}}

\def\vlambda{{\bm{\lambda}}}
\def\vphi{{\bm{\phi}}}

\def\vh{{\bm{h}}}

\def\vw{{\bm{w}}}
\def\vx{{\bm{x}}}

\def\mD{{\bm{D}}}
\def\mE{{\bm{E}}}

\def\mI{{\bm{I}}}

\def\mK{{\bm{K}}}
\def\mL{{\bm{L}}}
\def\mM{{\bm{M}}}

\def\mW{{\bm{W}}}
\def\mX{{\bm{X}}}

\def\mZ{{\bm{Z}}}

\def\mPhi{{\bm{\Phi}}}

\DeclareMathAlphabet{\mathsfit}{\encodingdefault}{\sfdefault}{m}{sl}
\SetMathAlphabet{\mathsfit}{bold}{\encodingdefault}{\sfdefault}{bx}{n}

\def\gD{{\mathcal{D}}}

\def\gG{{\mathcal{G}}}

\def\gJ{{\mathcal{J}}}

\def\gL{{\mathcal{L}}}

\def\gN{{\mathcal{N}}}

\def\gP{{\mathcal{P}}}

\def\sE{{\mathbb{\xi}}}

\def\sI{{\mathbb{I}}}

\def\sV{{\mathbb{V}}}

\newcommand{\R}{\mathbb{R}}

\DeclareMathOperator*{\argmax}{arg\,max}
\DeclareMathOperator*{\argmin}{arg\,min}

\def\<#1,#2>{\langle #1,\,#2\rangle}

\usepackage[sort,compress]{cite}
\usepackage[english]{babel}
\usepackage{url}
\usepackage{xcolor}
\newcommand\MYhyperrefoptions{bookmarks=true,bookmarksnumbered=true,
pdfpagemode={UseOutlines},plainpages=false,pdfpagelabels=true,
colorlinks=true,linkcolor={black},citecolor={black},urlcolor={black},
pdftitle={Local Neighborhood and Graph Construction using Non-Negative Kernel Regression},
pdfsubject={IEEE Journal},
pdfauthor={Sarath Shekkizhar, Antonio Ortega},
pdfkeywords={Nearest neighbors, graph construction, k-nearest neighbor, $\epsilon$-neighborhood, local neighborhood.}}
\usepackage[\MYhyperrefoptions,pdftex]{hyperref}
\addto\extrasenglish{
    
}
\usepackage{enumitem}
\usepackage{rotating}
\usepackage{tabulary}
\usepackage{multirow}
\usepackage[ruled, lined]{algorithm2e}

\usepackage{graphicx, epstopdf}
\usepackage{subcaption}
\graphicspath{{./figs/}, {./figs/nnk_graphs/}}

\newtheorem{theorem}{Theorem}

\newtheorem{lemma}{Lemma}[theorem]
\newtheorem{proposition}{Proposition}
\newtheorem{corollary}{Corollary}[theorem]
\theoremstyle{remark}
\newtheorem{remark}{Remark}

\usepackage{booktabs} 

\begin{document}
%
\title{Neighborhood and Graph Constructions using \\ Non-Negative Kernel Regression}
\author{Sarath~Shekkizhar,
        ~Antonio~Ortega
\IEEEcompsocitemizethanks{\IEEEcompsocthanksitem S.Shekkizhar, A.Ortega are with the Department
of Electrical and Computer Engineering, University of Southern California, Los Angeles,
CA.\protect\\
E-mail: \{shekkizh, aortega\}@usc.edu
}
}

\IEEEtitleabstractindextext{%
\begin{abstract}
Data-driven neighborhood and graph constructions are often used in  machine learning and signal processing applications.
Among these, 
$\rk$-nearest neighbor~($\rk$NN) and $\epsilon$-neighborhood methods are the most commonly used methods due to their computational simplicity, even though they often involve a somewhat ad hoc selection of their parameters, 
i.e., $\rk$ and $\epsilon$. We make two main contributions in this paper.  First, we present an alternative view of neighborhoods, where we show that neighborhood construction is equivalent to a sparse signal approximation problem. Second, we propose an algorithm, non-negative kernel regression~(NNK), for obtaining neighborhoods that lead to better sparse representation. NNK draws similarities to the orthogonal matching pursuit approach to signal representation and possesses desirable geometric and theoretical properties. Experiments show (i) the robustness of the NNK algorithm for neighborhood and graph construction, (ii) its ability to adapt  the number of neighbors to the data properties, and (iii) its superior performance in neighborhood and graph-based machine learning tasks. 
\end{abstract}

\begin{IEEEkeywords}
Nearest neighbors, Graph construction, k-nearest neighbor, $\epsilon$-neighborhood, Local neighborhood, Kernel regression.
\end{IEEEkeywords}}

\maketitle

\IEEEdisplaynontitleabstractindextext

\IEEEpeerreviewmaketitle

\IEEEraisesectionheading{\section{Introduction}\label{sec:introduction}}
\IEEEPARstart{A} simple pattern recognition scheme, whose application has traces dating back to the $11$th century \cite{pelillo2014alhazen}, is that of the nearest neighbors rule \cite{fix1952discriminatory,cover1967nearest}: data points that are close will have properties that are similar. 
Formally, given data points as vectors $\{\vx_1, \vx_2 \dots \vx_N\} \in \R^d$ and a \textit{query} $\vx_q$, defining the best neighborhood to represent $\vx_q$  involves  \emph{selection} of a subset of the $N$ points~(neighbors) followed by \textit{weight assignment} to each of those neighbors. 
$\rk$-nearest neighbor ($\rk$NN) \cite{nadaraya1964estimating, watson1964smooth} and $\epsilon$-neighborhood are among the most popular local neighborhood approaches used in practice, with applications in density estimation, classification, and regression \cite{biau2015lectures}.
These methods define a local neighborhood based on a parameter choice for \textit{neighborhood selection}, namely  the number of nearest neighbors~($\rk$) in   $\rk$NN or the maximum distance of the neighbors from a query~($\epsilon$) in $\epsilon$-neighborhood.  
Each selected neighbor is then given a positive weight using a kernel that quantifies its influence on the query (\textit{weight assignment}).

Local \emph{neighborhood definition or construction} is often the starting point for graph-based algorithms in scenarios where no graph is given {\em a priori} and a graph has to be \emph{constructed} to fit the data. 
A popular approach in these cases is to start by constructing a \emph{directed} graph using the neighborhoods and weights provided by $\rk$NN and $\epsilon$-neighborhood, leading to $\rk$NN-graphs and $\epsilon$-graphs. 
If necessary an undirected graph can be obtained from the directed graph. 
Note that graph construction is the first step in several graph signal processing \cite{ortega2018, ortega2022introduction} and graph-based learning methods \cite{Hamilton2017, chami2022machine}.  Consequently, the quality~(optimality, robustness, sparsity) of the graph representation is crucial to the success of these downstream algorithms \cite{Maier, DeSousa2013}. 

Even though techniques for neighborhood selection based on the choice of a single optimal value for $\rk$ or $\epsilon$ \cite{stone1977consistent, devroye2013probabilistic} exist, they 
can fail in non-uniformly distributed datasets where it might be desirable to adapt the number of neighbors  ($\rk/\epsilon$) to the local characteristics of the data.
Several approaches have been proposed to address neighborhood adaptivity: \cite{wettschereck1994locally} introduced a cross-validation method to select $\rk$ locally; \cite{baoli2004adaptive, balsubramani2019adaptive} defined $\rk$ using a class population-based heuristic; while \cite{ghosh2007nearest,mullick2018adaptive} optimize for $\rk$ using a Bayesian and neural network-based learning model. 
%
These methods focus on the selection component of neighborhood definition and do not address the weighing of the neighbors.
Recently, \cite{anava2016k} 
proposed an algorithm ($\rk^*$NN) for solving both the selection and weighting under the assumption of smoothness in functions defined on the data.

However, all of these adaptive approaches to neighborhood definition \cite{wettschereck1994locally,baoli2004adaptive, balsubramani2019adaptive,ghosh2007nearest,mullick2018adaptive,anava2016k} can only be applied to labeled data settings, where labels 
can be used for (hyper)~parameter selection. 
Since  no extensions are available for unlabelled data, which is a typical scenario for neighborhood definitions \cite{biau2015lectures, chen2018explaining}, this severely restricts their application.
Further, a shortcoming common to all existing approaches is their limited geometrical interpretation: they only consider the distance (or similarity) between the query and the data and ignore the relative position of the neighbors themselves.
As an example, two data points $i$ and $j$ at distance $d$ to the query may be included in the local neighborhood irrespective  of whether $i$ and $j$ are far or close to each other. In contrast, our proposed method takes into account the similarity between the neighbors $i$ and $j$ and can remove ``geometrically redundant'' neighbors, leading to better  neighborhood definitions.

\begin{figure}[htbp]
    \centering
    \includegraphics[width=0.45\textwidth]{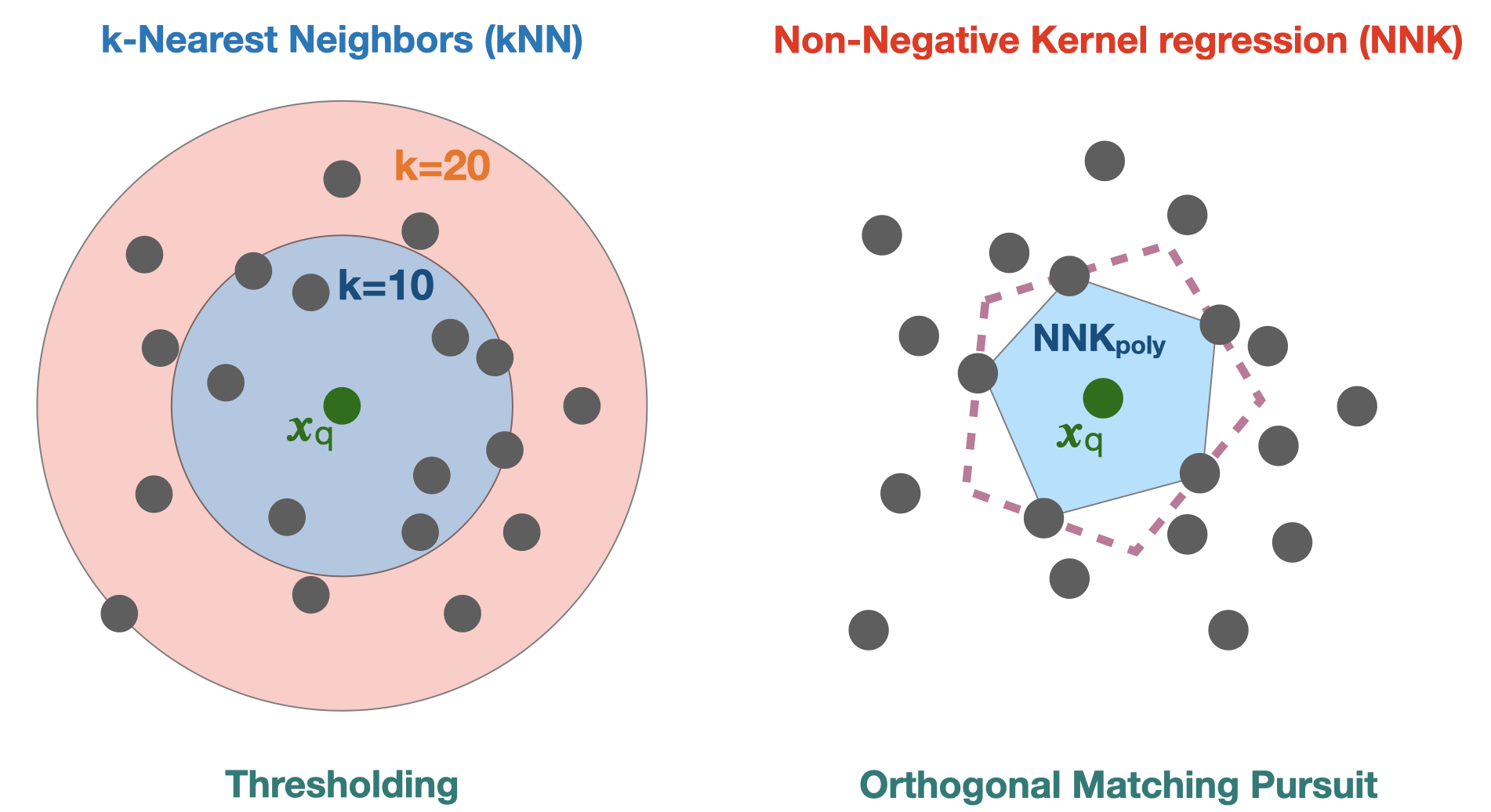}
    \caption{Geometric difference in neighborhood definitions using $\rk$-nearest neighbors ($\rk$NN) and the proposed NNK approach with a Gaussian kernel.  $\rk$NN selects neighbors based on the choice of $\rk$ and 
    can lead to the selection of  neighbors that are geometrically redundant, i.e., the vector from the query to neighbors is almost \emph{collinear} with other vectors joining the query and neighbors. In contrast, the proposed NNK definition
    selects only neighbors that correspond to non-overlapping directions.}
    \label{fig:knn_nnk_comparison}
\end{figure}

A key contribution of this work is a novel interpretation of the neighborhood definition 
as a non-negative sparse approximation of the query $\vx_q$ using data points $\{\vx_1 \dots \vx_N\}$. This view of neighborhoods allows one to analyze the problem using  tools from sparse signal processing \cite{donoho2006compressed} and opens up several directions for further research.
In particular, 
we show that $\rk$NN and $\epsilon$-neighborhood are signal approximation based on thresholding, with their corresponding hyperparameters, $\rk$ and $\epsilon$, used to control the sparsity. Our work is motivated by the well-known limitations of thresholding, which is only optimal when the candidate representation vectors (in our case the neighbors) are orthogonal. 

In this work, we leverage the sparse signal representation perspective to 
(i) establish a new notion of 
optimality for neighborhood definition and (ii) an improved algorithm, non-negative kernel regression (NNK), based on this optimality criterion. 
The sparse representation optimality criterion requires approximation errors to be orthogonal to the candidate representation vectors. We show that when applied to our problem this criterion  prevents the selection of  points that are ``geometrically redundant''. 
%
This idea is illustrated in \Figref{fig:knn_nnk_comparison}. Geometrically, the proposed NNK neighborhoods can be viewed as constructing a polytope around the query using points closest to it but eliminating those candidate neighbors that are further away along a \emph{similar} direction as an already selected neighbor. 
Thus, instead of selecting the number of neighbors based on pre-set parameters, e.g., $\rk$ or $\epsilon$, the NNK neighborhood is adaptive to the local data geometry and results in a principled \emph{neighbor selection and weight assignment}.

\subsection{Contributions and organization of the paper}

This paper substantially extends our earlier work, a shorter conference publication  \cite{shekkizhar2020}, in two ways. 

First,  
it presents a thorough theoretical analysis of the geometry, sparse representation, and optimization involved in the NNK framework, where we 
\begin{itemize}[leftmargin=*]
    \item further formalize with proofs our geometric characterization of NNK neighbors in input and kernel space via our kernel ratio interval (KRI) theorem 
    (\Secref{sec:geometry_interpret}),
    \item demonstrate the equivalence between NNK geometry and non-negative basis pursuit algorithms (\Secref{subsec:basis_pusuit_geometry}), and
    \item make explicit the connection between NNK and previously studied LLE-based neighborhoods \cite{Wang2008} and $\ell_1$-constrained pursuit methods  (Sections \ref{subsec:lle_connection},\ref{subsec:lasso});
\end{itemize}

Second, it offers a comprehensive experimental evaluation of the proposed NNK method in neighborhood and graph-based machine learning problems, where we 
\begin{itemize}[leftmargin=*]
    \item use $70$ OpenML classification datasets to show that NNK neighborhoods produce robust local classifiers outperforming those of  $\rk$NN, $\rk^*$NN \cite{anava2016k} (\Secref{sec:nnk_neighbor_experiments}), 
    \item quantify the scalability and sparsity of NNK graph construction relative to $\rk$NN, AEW \cite{Karasuyama2017}, LLE \cite{Wang2008}, and Kalofolias \cite{kalofolias2018large} graph constructions (\Secref{sec:sparsity_runtime_graph}), and
    \item validate that NNK graphs lead to better performance in graph-based semi-supervised learning and dimensionality reduction~(\Twosecrefs{sec:graph_ssl}{sec:lap_eig}). %
\end{itemize}
In other publications we have used the properties of NNK to revisit and improve methods in image processing \cite{shekkizhar2020efficient}, data summarization \cite{shekkizhar2021nnkmeans}, machine learning \cite{shekkizhar2021revisiting, shekkizhar2021model}, and geometric evaluation of deep learning models \cite{bonet2021channelredundancy, bonet2021channelearlystopping, SSL_geometry}. While these complementary works have reinforced the benefits of NNK, the present work provides the theoretical background and fundamental concepts behind the NNK formulation.

The rest of the paper is organized as follows: \Secref{sec:background} introduces notations, background, and related work. We discuss the connection between neighborhood definition and basis pursuit for sparse signal approximation in \Secref{sec:approximation_view} and introduce our NNK algorithm in \Secref{sec:nnk_algorithm}. We then derive the geometrical properties of NNK neighbors in \Secref{sec:geometry_interpret}. We conclude with experimental validation and discussion of the NNK framework in \Twosecrefs{sec:experiments}{sec:conclusion}. The proofs for all statements can be found in the Appendix.

\section{Background}
\label{sec:background}

We start by introducing the notation 
and two key components, kernel similarity and local linearity of data, that lead to distinct categories of existing neighborhood and graph constructions \cite{Qiao2018}. 
We also briefly review related work.

\subsection{Notation}
\label{sec:notations}
Throughout the paper, we use lowercase (e.g., $x$ and $\theta$), lowercase bold (e.g., $\vx$ and $\vtheta$), and uppercase bold (e.g., $\mM$ and $\mPhi$) letters to denote scalars, vectors, and matrices, respectively. We reserve the use of  $\mK$ for representing the matrix of the kernel function so that the kernel evaluated between two data points $i$ and $j$ is written $\mK_{i,j} = \kappa(\vx_i,\vx_j)$. 
Given a subset of indices $S$, $\vx_{S}$ denotes the subvector of $\vx$ obtained by taking the elements of $\vx$ at locations in $S$. A submatrix $\mM_{S,S}$ can be obtained similarly from $\mM$. The set complement $\bar{S}$ corresponds to the set of indices not in set $S$. 
We depict the vectors of all zeros and all ones as $\vzero, \vone$ respectively. 
The indicator function is represented using $\sI: \{\text{False},\text{True}\} \rightarrow \{0,1\}$ and the Hadamard product of matrices is denoted by  $\odot$ operator.

A graph $\gG = (\sV, \sE)$ is a collection of nodes indexed by the set $\sV = \{1, \ldots, N\}$ and connected by edges $\sE = \{(i, j,w_{ij})\}$, where $(i, j,w_{ij})$ denotes an edge of weight $w_{ij} \in \R^+$ between nodes $i$ and $j$. The weighted adjacency matrix $\mW$ of the graph is an $N \times N$ matrix with $\mW_{ij} = w_{ij}$. 
The combinatorial Laplacian of a graph is defined as $\mL = \mD - \mW$, where $\mD$ is the diagonal degree matrix with $\mD_{ii} = \sum_{j}\mW_{ij}$. The symmetric normalized Laplacian is defined as $\gL = \mI - \mD^{-1/2}\mW\mD^{-1/2}$.

\subsection{Similarity kernels} 
Kernels have a wide range of  applications in machine learning \cite{hofmann2008kernel, alvarez2011kernels}. 
As with most kernel-based works, we focus on kernels satisfying 
Mercer's theorem: a kernel function evaluated between two points corresponds to an inner product in a transformed space, the Reproducing Kernel Hilbert space (RKHS)\cite{mercer1909, aronszajn1950theory}, i.e., $\kappa(\vx_i, \vx_j) = \vphi_i^\top\vphi_j$, where $\vphi_i, \vphi_j$ are the kernel space representations of $\vx_i, \vx_j$\footnote{With a slight abuse of notation we use inner products as if the kernel representations $\vphi$ are real vectors. This allows us to use a common notation without considering the specific choice of kernel. In particular, our statements do generalize to the RKHS setting with continuous functions and inner products defined as $<f, g>=\int f(u)g(u)du$.}. 
Similarity kernels are used for assigning weights to selected neighbors in a neighborhood/graph definition. 
The choice of the kernel is, in general, task- or domain-specific with  some predefined kernels widely used in practice. 
Alternatively, one can also learn a kernel based on available data \cite{kulis2012metric, wang2015survey}.
Examples of data-driven kernel learning methods for neighborhood definition include  \cite{kapoor2005hyperparameter} and  adaptive edge weighting~(AEW) \cite{Karasuyama2017}. It should be emphasized that the learned kernels in these methods are used with $\rk$ or 
$\epsilon$-neighborhood selection and 
thus only offer a solution for the weight assignment problem. 

In this paper we primarily use the Gaussian kernel:
    \begin{align}
    \kappa_{\sigma^2}(\vx_i,\vx_j) = e^{-\frac{\|\vx_i - \vx_j\|^2}{2\sigma^2 }}, 
    \label{eq:gaussian_kernel}
    \end{align}
where $\sigma$ is the  bandwidth~(variance) of the kernel. 
We emphasize that the statements and algorithms presented here, unless stated otherwise, are applicable \textit{to all symmetric and positive-definite kernels, including those learned from data}.


\subsection{Locally linear neighborhoods}
High-dimensional data is often assumed to lie on or near a smooth manifold and thus can be approximated using locally linear patches. 
This is the principal assumption 
behind local linear embedding~(LLE)  \cite{Roweis2323}, which relies on  a local regression objective to
approximate a query, i.e.,
\begin{align}
\vtheta_S &= \argmin_\vtheta \;
\|\vx_q - \mX_S\vtheta\|^2_{_2}, \label{eq:lle_objective}
\end{align}
where $\mX_S$ is the matrix containing the $\rk$-nearest neighbors of $\vx_q$,  whose indices are denoted by set $S$.
A solution to 
\plainref{eq:lle_objective} can produce both positive and negative values 
for the entries of $\vtheta$, which would be the weights associated with the  elements in the set $S$. 
Thus, \cite{Wang2008} reformulated this problem for defining neighborhoods and graphs by introducing non-negativity constraints on the weights. 

Subsequent modifications to \cite{Wang2008} for neighborhood construction include: post-processing 
to enforce regularity \cite{Jebara2009}; formulating a \emph{global} objective to construct symmetrized graphs \cite{Daitch2009}; and additional regularizations in the objective to facilitate robust optimization \cite{Cheng2010, Han2016StructureAL}. 
Alternative approaches such as \cite{kalofolias2018large, kalofolias2016learn} can be considered as a reformulation of the LLE objective from a graph signal processing perspective with appropriate regularizers. 

\begin{table}[t]
    \centering
    \includegraphics[width=0.49\textwidth]{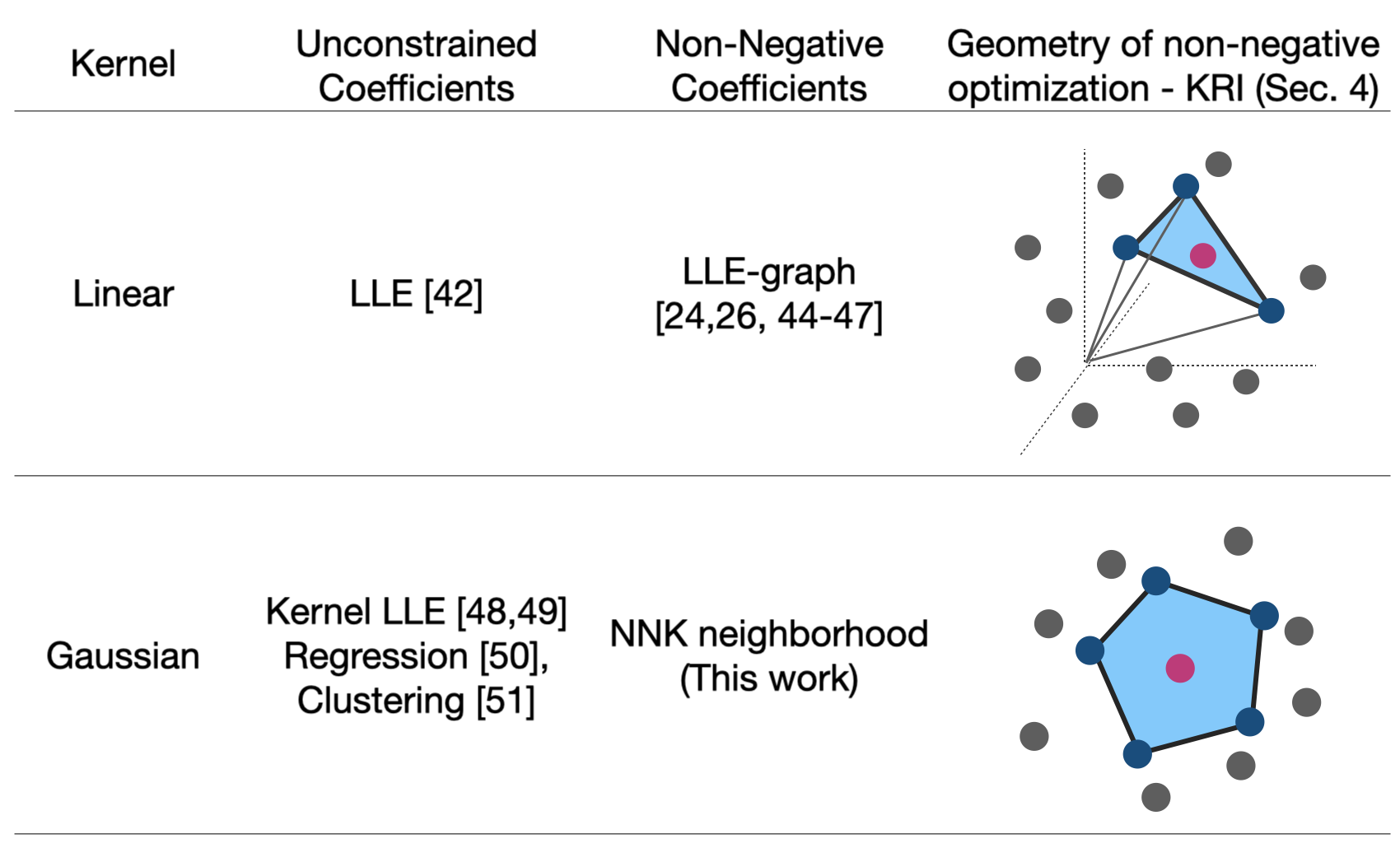}
    \caption{Contributions of our proposed method (NNK) in relation to previous works. NNK combines the concepts of non-negativity, kernels, and locally linear neighborhoods. Earlier works did not explore the sparsity or the geometric consequences of non-negativity~(why a data point $i$ is selected while another $j$ is not selected for representing a query $q$) in their definitions. The last column demonstrates the geometry of the neighborhood obtained, based on \autoref{th:KRI}, for Linear and Gaussian kernels ({\color{gray}Dataset}, {\color{magenta}Query}, {\color{blue}Neighbors} obtained with non-negative optimization).}
    \label{fig:lle_differences}
\end{table}

\subsection{Related work}
\label{subsec:related_work}
The proposed NNK framework reduces to non-negativity constrained LLE algorithms under a specific choice of  \emph{cosine kernel}~(see \Secref{subsec:lle_connection}).
We note that kernelized versions of LLE have been previously studied for dimensionality reduction \cite{ham2004kernel, kong2012iterative}, ridge regression \cite{Murphy2012}, subspace clustering \cite{PatelKernelSSC2014}, and matrix completion \cite{GiannakisMCE2019}. However, we emphasize that these methods did not (1) develop a connection between neighborhood construction and sparse signal approximation, or (2) study the kernelized objective under non-negativity constraints for the purpose of  defining neighborhoods. 
The lack of non-negativity in the kernelized formulations in these related works prevented them from developing  geometrical and theoretical properties of their solutions, which are possible for  those developed in our framework. 
We summarize these differences in \Tabref{fig:lle_differences}.
To our knowledge, the view of neighborhoods as a sparse signal representation has not been made explicit previously and we believe this perspective can lead to new ideas and even help improve related problems in data analysis. 
On a different note, works such as \cite{Giannakis2017} solve a complementary problem where the graph is given and kernels based on the graph are constructed for solving machine learning tasks. These methods align with, and reinforce, our intuition that neighborhoods and graphs are structures of similarity defined on a transformed space.

\section{Neighborhood definition: A sparse signal approximation view}
\label{sec:approximation_view}
In this section, we first formulate the neighborhood definition problem as a signal approximation. We then show that the $\rk$NN/$\epsilon$-neighborhood methods are equivalent to thresholding-based sparse approximation, and introduce a better representation approach with our NNK framework. 

\subsection{Problem formulation}
\label{sec:problem_formulation}
Given a dataset of $N$ data points $\{\vx_1, \vx_2, \dots, \vx_N \}$, the problem of neighborhood definition for a query $\vx_q$ is that of obtaining a weighted subset of the given $N$ data points that best represents the query.
The weights in a neighborhood definition are required to be non-negative, as negative weights would imply a \emph{neighbor} with negative influence or dissimilarity which are not meaningful in typical neighborhood-based tasks, e.g., density estimation \cite{biau2015lectures} or label propagation \cite{Zhu-SSL-ICML-03}. 
The use of non-negative constraints, either implicitly or explicitly, is a common feature in all neighborhood definitions. 

Neighborhood definition for a query can be viewed as a sparse signal approximation problem with non-negative  coefficients. 
The goal is to approximate a signal $\vphi_q$ in a kernel space, corresponding to the query $\vx_q$, as a sparse weighted sum of functions or signals~(called \emph{atoms}) corresponding to the given $N$ data points, i.e.,
\begin{align*}
    \vphi_q \approx \mPhi\vtheta
\end{align*}
where $\mPhi = [\vphi_1 \; \vphi_2 \; \dots \;\vphi_N]$ is the \emph{dictionary} of atoms based on the $N$ data points and  $\vtheta$ is a sparse non-negative vector. 

\subsection{$\rk$NN and $\epsilon$-neighborhood}
\label{subsec:basis_intepret}
We first analyze standard approaches such as $\rk$NN/$\epsilon$-neighborhood from the signal approximation perspective.
A $\rk$NN or $\epsilon$-neighborhood approach uses the kernel weights as edge weights. Under the distinction of data and transformed space, a kernel value can be interpreted as the correlation between two data points, i.e., for a query $q$,
\begin{align}
    \vphi_q^\top\vphi_j = \kappa(\vx_q, \vx_j) = \mK_{i,j}.
    \label{eq:kernel-inner-product}
\end{align}

Thus, $\rk$NN and equivalently $\epsilon$-neighborhood methods, correspond to a {\em thresholding} technique for basis selection wherein atoms corresponding to the $\rk$ largest correlations are selected from $\mPhi$ to approximate the signal $\vphi_q$. 
This approach is reminiscent of early methods for basis pursuit using thresholding \cite{daubechies1992ten}, i.e., selecting and weighing atoms based on the magnitude of the correlation between atoms and the signal. This strategy is optimal only when the dictionary is orthogonal~($\mPhi^\top\mPhi = \mI$), and is sub-optimal for over-complete or non-orthogonal dictionaries. In our problem setting, the dictionary $\mPhi$ is typically not orthogonal since the atoms  (i.e., the neighbors to the query) need not have zero correlation, i.e., in general, $\vphi_i^\top\vphi_j > 0$. 

\subsection{Non-Negative Kernel Regression Neighborhood}
\label{sec:nnk_algorithm}
We now propose an improved neighborhood definition, non-negative kernel regression~(NNK), based on our observation of the sub-optimality (in terms of approximation) of $\rk$NN/$\epsilon$-neighborhood. 
NNK can be viewed as a sparse representation approach where the coefficients $\vtheta$ for the representation are such that the error in representation is orthogonal to the space spanned by the selected atoms. 
We note that this reformulation results in neighborhoods that are adaptive to the local geometry of the data and robust to parameters such as $\rk/\epsilon$. For a query $\vx_q$ the NNK weights $\vtheta$ are found by solving:
\begin{align}
\vtheta_S = & \argmin_{\vtheta \colon \vtheta \geq \vzero} \;
\|\vphi_q - \mPhi_S\vtheta\|^2_{_2}, \label{eq:nnk_lle_objective}
\end{align}
where $\mPhi_S$ corresponds to the kernel space representations of an approximate set of  data point candidates $S$ for the neighborhood, such as the $\rk$ nearest neighbors. 

Now, employing the kernel trick, given \plainref{eq:kernel-inner-product} the objective in (\ref{eq:nnk_lle_objective}) can be rewritten as the minimization of:  
\begin{align}
\gJ_q(\vtheta) &= \frac{1}{2}\|\vphi_q - \mPhi_S\vtheta\|^2_{_2} \nonumber\\
&= \frac{1}{2}\vtheta^\top\mK_{S,S}\vtheta - \mK^\top_{S,q}\vtheta + \frac{1}{2}\mK_{q,q}. \label{eq:kernelized_objective}
\end{align} Consequently, the NNK neighborhood is obtained as the solution to the optimization problem, 
\begin{align}
\vtheta_S &= \argmin_{\vtheta\colon\vtheta \geq \vzero} \;
\frac{1}{2}\vtheta^\top\mK_{S,S}\vtheta - \mK^\top_{S,q}\vtheta,   \label{eq:neighborhood_nnk_objective}
\end{align}
with non-zero elements of the solution $\vtheta_S$ corresponding to the selected neighbors and weights given by the corresponding values in $\vtheta$.

\begin{algorithm}[t]
\SetAlgoNoLine
\DontPrintSemicolon
\SetKwInOut{Input}{Input}\SetKwInOut{Output}{Output}
\Input{Query $\vx_q$, Data $\gD=$\{$\vx_1 \cdots \vx_N$\}, Kernel function $\kappa$, No. of neighbors $\rk$}
$S$ = \{k nearest neighbors of query $\vx_q$\}\;
$\vtheta_{S}  = \underset{\vtheta \geq \vzero}{\argmin} \;
\frac{1}{2}\vtheta^\top\mK_{S,S}\vtheta - \mK^\top_{S,q}\vtheta $\;
\BlankLine
$P$ = \{ elements of $S$ corresponding to $\sI(\vtheta_{S} > \vzero)$ \}\;
\Output{Neighbors $P$,\hspace{0.5em} Neighbor weights $\vtheta_{P}$}
\caption{NNK Neighborhood algorithm}
\label{tab:nnk_algorithm}
\end{algorithm}

As in other kernel-based learning methods \cite{Scholkopf2001}, the NNK objective \plainref{eq:neighborhood_nnk_objective} does not require the explicit kernel space representations and needs only knowledge of the similarity, the kernel matrix $\mK$, of a subset $S$. 
The pseudo-code for the proposed method is presented in \autoref{tab:nnk_algorithm}.

The NNK framework binds principles that have been key ingredients to neighborhood definitions, 
namely,  kernels and local linearity. 
Unlike earlier methods which used kernels (predefined or learned) only to define the weights for selected neighbors, the NNK neighborhood can be viewed as identifying a small number of non-negative regression coefficients (as in LLE) for an approximation in the representation space associated with a kernel (see \Tabref{fig:lle_differences}). 

\subsubsection{NNK and constrained LLE objective}
\label{subsec:lle_connection}
Previously studied LLE~\cite{Wang2008} is a specific case of the proposed NNK neighborhood algorithm with a kernel similarity defined in the input data space. 
\begin{proposition}
\label{prop:lle_generalization}
The NNK algorithm with a query-dependent cosine similarity~(\ref{eq:lle_cosine_kernel}) reduces to the locally linear neighborhood definition.
\begin{align}
    \kappa_q(\vx_i, \vx_j) = \frac{1}{2} + \frac{(\vx_i - \vx_q)^\top(\vx_j - \vx_q)}{2\|\vx_i - \vx_q\|\|\vx_j - \vx_q\|} \label{eq:lle_cosine_kernel}
\end{align}
\end{proposition}


\subsection{Basis Pursuits}
\label{subsec:basis_pursuit}
In this section, we present basis pursuit \cite{chen1994basis} approaches to sparse signal approximation and their application for defining neighborhoods. 
We show alongside that the proposed NNK neighborhood is an efficient sparse approximation that leverages the given problem setup, and is equivalent to these pursuit methods under certain conditions. 

\subsubsection{Matching Pursuits}
\label{sec:mp_omp_neighborhoods}
Matching pursuits (MP) \cite{Mallat93} is a greedy algorithm for sparse approximation which iteratively selects atoms that have the highest correlation with the signal's approximation error at a given iteration. Variants of this method include orthogonal matching pursuits (OMP) \cite{tropp2007signal} and stagewise OMP (Block Selection + OMP) \cite{donoho2012sparse}, where the coefficients for the selected atoms are recalculated at each step so that the residue is orthogonal to the space spanned by the selected vectors. 
We now present the steps involved so that MP and OMP can be used in neighborhood definition and then establish the analogy between NNK and OMP. 

\noindent The first step in MP and OMP is the same, where we find 
\[      j_1 = \underset{j}{\argmax}\; \vphi^\top_q\vphi_j  = \mK_{q,j}.         
 \]
With this choice, we can compute the error or residue incurred by approximating $\vphi_q$ with a single vector $\vphi_{j_1}$ as:  
  \[      
        \vphi_{r_{1}} = \vphi_q - \mK_{q,j_1}\vphi_{j_1}.
        \]
Now, at a step $s$, we would find: 
        \[
            j_s = \argmax_{j \neq j_1, j_2 \dots j_{s-1}} \; \vphi^\top_j\vphi_{r_{s-1}} \]
Then, denote:
            \[
            \quad \mPhi_S = \left[\vphi_{j_1}\; \vphi_{j_2}\; \dots \;\vphi_{j_{s}}\right],
            \]     
        where the $S = \{j_1, \; \ldots \; j_s\}$ are the indices of all the selected atoms so far for representation. In MP, one assigns the correlation between the residue and the selected atom as weights for approximation. However, in OMP, we need to find weights $\vw_S$ associated with each selected basis such that the energy of the residue is minimized, i.e.,
        \begin{align}
            \vw_S &=  \underset{\vw\colon\vw \geq \vzero}{\argmin} \; \|\vphi_{r_S}\|^2 \notag \\
            &= \underset{\vw\colon\vw \geq \vzero}{\argmin}\; \mK_{q,q} - 2\mK_{S,q}^\top\vw_S +\vw_S^\top\mK_{S,S}\vw_S,
            \label{eq:omp_least_squares}
        \end{align}
        where $\vphi_{r_S} = \vphi_q - \mPhi_S\vw_S$ is the residue at step $s$. This ensures that the approximation error is orthogonal to the span of the selected atoms $\mPhi_S$.
        
Note that by constraining the weights to be positive, the span of selected bases at each step is a convex cone, while the residue is orthogonal to this cone~\cite{Gale51}. 
We see that solving for the weights $\vw_S$ in (\ref{eq:omp_least_squares}) is equivalent to the NNK objective (\ref{eq:neighborhood_nnk_objective}) given the set $S$. 
Thus, the proposed NNK algorithm (\autoref{tab:nnk_algorithm}) bypasses the \emph{greedy selection} with a pre-selected set of \emph{good} atoms avoiding expensive computation~(iterative selection and orthogonalization). 
We note that given a set of pre-selected atoms it is straightforward to perform an orthogonal projection, such as the one performed at each OMP step. 
Our problem makes it possible to perform this pre-selection using computationally efficient approximate neighborhood methods. Thus, we do not need to perform greedy selection and instead can use an approximate neighborhood method to choose the set $S$ without loss in performance, as long as $S$ is 
large enough (so that all candidates needed to form an optimal neighborhood are likely to be included).
\autoref{tab:nnk_mp_omp_algorithm} provides the pseudo-code for neighborhood definition based on MP and OMP with sequential neighbors selection. 

\begin{algorithm}[t]
\SetAlgoNoLine
\DontPrintSemicolon
\SetKwInOut{Input}{Input}\SetKwInOut{Output}{Output}
\Input{Query $\vx_q$, Data $\gD=$\{$\vx_1 \cdots \vx_N$\}, Kernel function $\kappa$, Maximum no. of neighbors k}

$j_1 = \argmax_j\; \mK_{q,j}$\;
$\vtheta_1 = \mK_{q, j_1}, \quad S = \{j_1\}$\;
\For{ $s=2, 3, \dots \rk$ }{
$j_s = \underset{j \notin S}{\argmax} \;  \mK_{q,j} - \mK^\top_{S,j}\vtheta_{s-1}$\;
\uIf{ $\mK_{q, j_s} - \mK^\top_{S,j_s}\vtheta_{s-1} < 0$ }{break}
\BlankLine
\uIf{ OMP }{$\vtheta_s = \underset{\vtheta\geq \vzero}{\argmin} \;             \frac{1}{2}\vtheta^\top\mK_{S,S}\vtheta - \mK^\top_{S,q}\vtheta$}
\uElse{$\vtheta^\top_s = \left[\vtheta_{s-1}^\top \;\; (\mK_{q, j_s} - \mK^\top_{S,j_s}\vtheta_{s-1}) \right]$}
\BlankLine
$S = S \cup \{j_s\}$\;
}
\BlankLine
\Output{Neighbors $S$,\hspace{0.5em} Neighbor weights $\vtheta_{S}$}
\caption{MP and OMP Neighborhood algorithms}
\label{tab:nnk_mp_omp_algorithm}
\end{algorithm}

\subsubsection{$l_1$ regularized pursuits}
\label{subsec:lasso}
Another set of methods for solving sparse approximation problems involves a convex relation where one \emph{softens} the sparsity constrain~($l_0$ norm of reconstruction coefficients) with an $l_1$ norm constrain \cite{efron2004least,daubechies2010iteratively}. These approaches, also known as least absolute shrinkage and selection operator~(LASSO) regression \cite{tibshirani1996regression}, allow for the use of optimization techniques such as linear programming to obtain an approximate sparse solution.
A neighborhood definition at a query $q$ using a $l_1$ regularized basis pursuit  corresponds to an  optimization problem of the form
\begin{align*}
\vtheta_{\gD} &= \argmin_{\vtheta\colon\vtheta \geq \vzero} \;
\frac{1}{2}\vtheta^\top\mK_{\gD, \gD}\vtheta - \mK^\top_{\gD,q}\vtheta +\eta\|\vtheta\|_1,
\end{align*}
where $\gD$ is the set of all $N$ data points and $\eta$ is the Lagrangian hyperparameter.
Given the non-negativity constraint on the coefficients in our setting, we can replace the $l_1$ norm by the sum, leading to: 
\begin{align}
\vtheta_{\gD} &= \argmin_{\vtheta\colon\vtheta \geq \vzero} \;
\frac{1}{2}\vtheta^\top\mK_{\gD,\gD}\vtheta - \mK^\top_{\gD,q}\vtheta +\eta\vtheta^\top\vone .\label{eq:nnk_objective_l1}
\end{align}
\begin{proposition}
\label{prop:l1_simplification}
Let $\vtheta^*$ be the solution to objective (\ref{eq:nnk_objective_l1}). Then,
\begin{align}
    \forall j \colon \mK_{q,j} < \eta \qquad  \vtheta^*_{j} = 0 
\end{align}
\end{proposition}
Proposition \ref{prop:l1_simplification} implies that the set $\bar{S}=\{i\colon \mK_{q,i} < \eta\}$ can be safely removed from the optimization since $\vtheta^*_{\bar{S}} = 0$. Thus objective (\ref{eq:nnk_objective_l1}) is equivalent to the NNK objective  \plainref{eq:neighborhood_nnk_objective} with set $S$, such that $\gD = S \cup \bar{S}$ and $S \cap \bar{S} = \emptyset$, i.e., $S$ is the set of data points with similarity greater than $\eta$. 
Thus, choosing $S$ as a set of $\rk$NN neighbors to initialize the objective \plainref{eq:neighborhood_nnk_objective} provides an optimal solution to \plainref{eq:nnk_objective_l1} with parameter  $\eta$. 


\subsection{Graph construction}
\label{subsec:nnk_graph}

The NNK neighborhood definition in \Secref{sec:nnk_algorithm} and alternative basis pursuit approaches in \Secref{subsec:basis_pursuit} can be adapted for graph construction 
by solving for the neighborhood at each data point $\vx_i$ 
%
to first obtain a directed graph adjacency matrix $\mW$, namely, $\mW_{S, i} = \vtheta_S$ and $\mW_{\bar{S}, i}=\vzero$. 
To obtain an undirected graph, we observe that an edge conflict may occur, say between nodes $i$ and $j$, due to the difference in the local neighborhood of the nodes, i.e., the influence of $i$ on $j$ depends on the neighborhood of $j$ and vice versa\footnote{Note that the directed nature of NNK is not surprising, since graph construction from $\rk$NN neighborhoods, which are used for initializing set $S$ in NNK, are also not symmetric.}. 
In such a scenario, we consider the approximation error corresponding to nodes $i$ and $j$, namely $\gJ_i(\vtheta)$ and $\gJ_j(\vtheta)$ from \plainref{eq:kernelized_objective}, and keep the weight $w_{ij}$ corresponding to the smallest of the two objectives (i.e., $q=i$ and $q=j$).
We note that the proposed symmetrization, though intuitive from a representation perspective, 
requires further study and we leave its formulation and connection to downstream performance as open problems for future work.
\autoref{tab:nnk_graph_algorithm} provides the pseudo-code for NNK graph construction. 

\begin{algorithm}[t]
\SetAlgoNoLine
\DontPrintSemicolon
\SetKwInOut{Input}{Input}\SetKwInOut{Output}{Output}
\Input{Data $\gD =$ \{$\vx_1 \cdots \vx_N$\}, Kernel function $\kappa$, No. of neighbors k}
\For{ $i=1, 2, \dots N$ }{
$S =$ \{ k nearest neighbors of $\vx_i$ in $\gD_{-i}$\}\;
$\vtheta_{S}  = \underset{\vtheta \geq \vzero}{\argmin} \;
\frac{1}{2}\vtheta^\top\mK_{S,S}\vtheta - \mK^\top_{S,i}\vtheta $\;
\BlankLine
$\gJ_i = \frac{1}{2}\vtheta^\top_S\mK_{S,S}\vtheta_S -
\mK^\top_{S,i}\vtheta_S + \frac{1}{2}\mK_{i,i}$ \;
\BlankLine
$\mW_{S, i} = \vtheta_{S}, \; \; \mW_{\bar{S}, i} = \mathbf{0}$ \;
$\mE_{S, i} = \gJ_i\:\vone, \; \; \mE_{\bar{S}, i} = \mathbf{0}$\;
}
$\mW = \sI(\mE \leq \mE^\top)\odot\mW + \sI(\mE > \mE^\top)\odot\mW^\top$\;
\BlankLine
\Output{Graph Adjacency $\mW$,\hspace{0.5em} Local error $\mE$}
\caption{NNK Graph algorithm}
\label{tab:nnk_graph_algorithm}
\end{algorithm}

\subsection{Complexity}
The proposed NNK method consists of two steps. 
First, we find $\rk$-nearest neighbors corresponding to each node. Although brute force implementation has complexity $O(N\rk d)$, there exist efficient algorithms to find approximate neighborhoods using sub-linear-time search with additional memory \cite{johnson2019billion}.
Second, a non-negative kernel regression (\ref{eq:neighborhood_nnk_objective}) is solved at each node. 
This objective is a constrained quadratic function of $\vtheta$ and can be solved efficiently using  structured programming methods that require $O(\rk^3)$ or, with more careful analysis, $O(\rk \hat{\rk}^2)$ complexity, where $\hat{\rk}$ is the number of neighbors 
with non-zero weights in the optimal solution. 
In summary, the NNK framework defines a neighborhood with an initialization based on $\rk$NN followed by a weight estimation with a runtime complexity that is, at most, cubic in the size of the initial neighborhood set $S$. 

In practice, as observed in our experiments (\Secref{sec:experiments}),  the added complexity of NNK, after initializing with $\rk$NN, is often negligible compared to the cost of $\rk$NN, given that $Nd >> \rk^2$. 
We note that the choice of $\rk$ for the initial set $S$ provides a  trade-off between runtime complexity and optimality: a small $\rk$ has lower complexity but results in a set $S$ that is not representative enough, while a larger $\rk$ has higher complexity but allows for obtaining the optimal neighborhood.   
For the graph construction problem, the runtime of the NNK optimization is $O(N\rk\hat{\rk}^2)$, in addition to the time required for $\rk$NN. Note that, due to the local nature of the optimization problem, the NNK algorithm can be executed in parallel and is typically limited by the complexity of finding a good set of initial neighbors.

\section{Geometric Interpretation}
\label{sec:geometry_interpret}
Unlike existing neighborhood definitions, such as $\rk$NN,  where a neighbor selection for a query $q$ involving two points $i$ and $j$
is driven solely based on the metric on $(q,i)$ and $(q,j)$, 
NNK also takes into account the relative positions of nodes $i$ and $j$ using the metric on $(i,j)$.
We now present a theoretical analysis of the geometry in NNK neighborhoods. 

\subsection{NNK and active constraint set condition}
In constrained optimization problems, some constraints will be strongly binding, i.e., the solution at some indices will be zero to satisfy the KKT condition of optimality~(outlined in \Secref{subsec:KKT}). These constraints are referred to as \textbf{active constraints}, knowledge of which helps reduce the problem for optimization and analysis using only the inactive subset. This is because any constraints that are active at a current feasible solution will remain active at the optimum \cite{boyd2004convex}.
\begin{proposition}
\label{prop:active_set}
The NNK optimization objective \plainref{eq:kernelized_objective} at a query $\vx_q$ satisfies active constraint conditions. 
Given a partition 
set $P$ such that $\vtheta_P > 0$ (inactive) and $\vtheta_{\bar{P}} = 0$ (active), the solution $[\vtheta_P \; \vtheta_{\bar{P}}]^\top$ is the optimal solution to NNK iff
\begin{align}
    \mK_{P, P}\vtheta_P &= \mK_{P, q} \label{eq:inactive_set_eq}\\
    \mK^\top_{P, \bar{P}}\vtheta_P - \mK_{\bar{P},q} &\geq \vzero \label{eq:active_set_eq}
\end{align}
\end{proposition}

Proposition~\ref{prop:active_set} allows us to analyze the neighbors obtained in the NNK framework, one pair at a time, as a data point that is zero weighted~(active constraint) will remain zero weighted at the optimal solution. 
We first introduce conditions for the existence of NNK neighbors in the form of the Kernel Ratio Interval~(KRI) theorem, which is  applied \emph{inductively} to unfold the geometry of NNK neighborhoods.

\subsection{Geometry of NNK neighbors}
\begin{figure}[htbp]
\centering
\includegraphics[height=0.15\textheight]{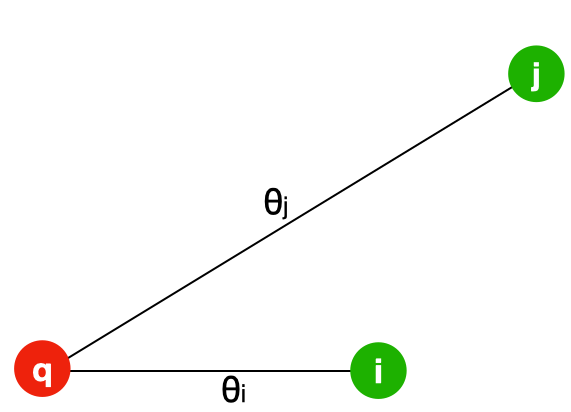}
\caption{Problem setup for a query $q$ with two points $i,j$.}
\label{fig:3node}
\end{figure}

\begin{theorem}\textbf{Kernel Ratio Interval:}
\label{th:KRI}
Given a scenario with a query $q$ and two data points, $i$ and $j$ (see  Fig~\ref{fig:3node}) and a similarity kernel with range $[0, 1]$, the necessary and sufficient condition for both $i$ and $j$ to be chosen as neighbors of $q$ by NNK is given by 
\begin{align}
\mK_{i,j} < \frac{\mK_{q,i}}{\mK_{q,j}} < \frac{1}{\mK_{i,j}}.  \label{eq:kernel_ratio_interval}
\end{align}
\end{theorem}
In words, Theorem~\ref{th:KRI} states that when $i$ and $j$ are very similar, it is less likely that \textit{both} will be chosen as neighbors of $q$, because the interval in \plainref{eq:kernel_ratio_interval}, which  defines when this can happen, becomes narrower.  
%
The KRI condition of (\ref{eq:kernel_ratio_interval}) does not make any assumptions on the kernel, other than that it be symmetric with values in $[0,1]$\footnote{The general form of KRI is presented in Appendix \ref{subsec:general_KRI}. We restrict to kernels $\in [0,1]$ in the main text for ease of understanding and to make the basis pursuit connection explicit. We note that the analysis and properties of NNK are applicable for all Mercer kernels.}. 

\begin{corollary}\textbf{Plane Property:}
\label{corollary:plane_property}
Each NNK neighbor corresponds to a hyperplane with normal in the  direction of the line joining the neighbor and query (\Figref{fig:plane}), points beyond which will not be NNK neighbors for the case of Gaussian kernel (\ref{eq:gaussian_kernel}).
\end{corollary}

\begin{figure}[htbp]
\centering
\begin{subfigure}{0.29\textwidth}
\centering
\includegraphics[width=\textwidth]{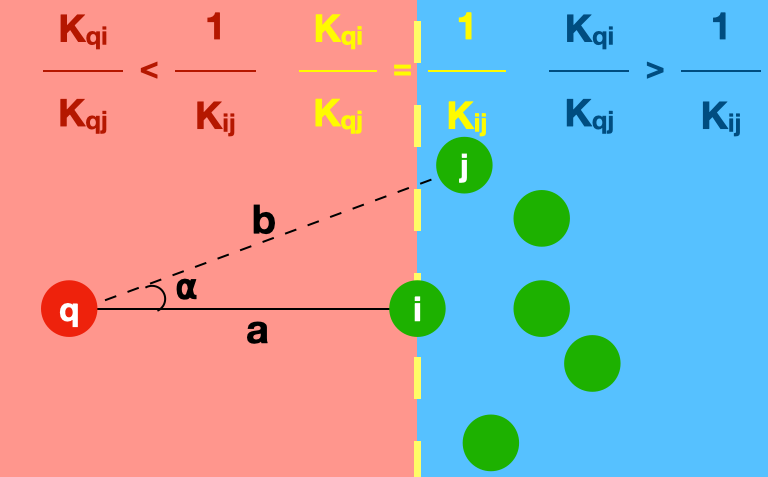}
\caption{}
\label{fig:plane}
\end{subfigure} 
~
\begin{subfigure}{0.18\textwidth}
\centering
\includegraphics[width=\textwidth]{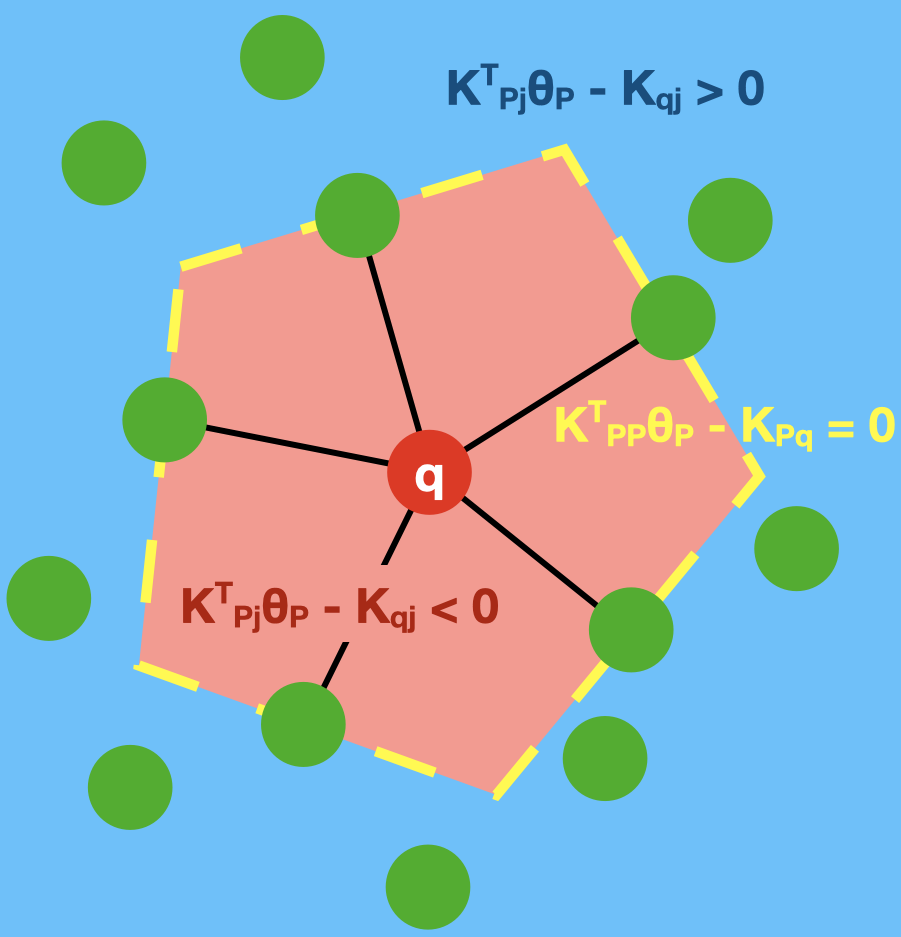}
\caption{}
\label{fig:polytope}
\end{subfigure}
\caption{Local geometry~(denoted in red) of NNK neighborhood for a kernel similarity proportional to the distance between the data points. \textbf{(a)} Hyper plane associated with a selected NNK neighbor. \textbf{(b)} Convex polytope corresponding to NNK neighbors around a query.}
\end{figure}

Corollary~\ref{corollary:plane_property} can be described in terms of projections as illustrated in \Figref{fig:plane}. Assume that query $q$ is connected to a neighbor $i$ ($\theta_{i} > 0$). Now, define a hyperplane that contains $i$ and is perpendicular to the line going from $q$ to $i$. Then any data point $j$ that is beyond this hyperplane (i.e., on the half space not containing $q$) will not be an NNK neighbor ($\theta_{j} = 0$).  This is because the orthogonal projection of ($\vx_j - \vx_q$) along the direction ($\vx_i - \vx_q$) is at a position beyond $\vx_i$. 

\begin{corollary}\textbf{Polytope Property (Local geometry of NNK):}
\label{corollary:polytope_optimality}
The local geometry of the NNK neighborhood with Gaussian kernel (\ref{eq:gaussian_kernel}), for a given query $q$, is a convex polytope around the query (\Figref{fig:polytope}). 
The optimal solution to the NNK objective \plainref{eq:neighborhood_nnk_objective} with $\vtheta = [\vtheta_P,\; \vtheta_{\bar{P}}]^\top$, where $\vtheta_P > 0$ and $\vtheta_{\bar{P}} = 0$ 
satisfies
\begin{enumerate}[label=(\Alph*)]
    \item $\mK^\top_{P, j}\vtheta_P - \mK_{q,j} \geq  0$ 
        $ \quad \iff \; \exists \; i \; \in \; P \; \colon \; \frac{\mK_{q,i}}{\mK_{q,j}} \geq \frac{1}{\mK_{i,j}}$
    \item  $\mK^\top_{P, j}\vtheta_P - \mK_{q,j} < 0$ 
        $\quad \iff \; \forall \; i \; \in \; P \; \colon \; \frac{\mK_{q,i}}{\mK_{q,j}} < \frac{1}{\mK_{i,j}} $
\end{enumerate}
\end{corollary}
The conditions in Corollary \ref{corollary:polytope_optimality} are a direct consequence of applying Corollary \ref{corollary:plane_property} to a series of points. Gathering successive conditions leads to forming a polytope, as illustrated in Fig.~\ref{fig:polytope}, and to the optimality conditions of the active set Proposition~\ref{prop:active_set}. 
In other words, the NNK neighborhood algorithm constructs a polytope around a query using selected neighbors while disconnecting geometrically redundant data points outside the polytope.
This suggests that the local connectivity of NNK neighbors is a function of the \emph{local dimension} of the data manifold \cite{hurtado2022study}.

We note that corollaries~\ref{corollary:plane_property} and \ref{corollary:polytope_optimality} are applicable to all kernels that are monotonic with respect to the distance between data points. However, we emphasize that equivalent geometric conditions based on the KRI can be obtained for other kernels. For example, in the case of linear kernels, it can be shown that the NNK geometry obtained is that of a convex cone with each selected neighbor corresponding to a hyperplane passing through the origin \cite{Gale51} (see \Figref{fig:lle_differences}). 

\subsection{NNK Geometry in kernel space}
In some cases, such as bio-sequences or text documents, for which it may be difficult to represent the input as explicit feature vectors, using a distance-based kernel, such as the Gaussian kernel in \plainref{eq:gaussian_kernel}, is not feasible, and alternative kernel similarity measures need to be employed \cite{lodhi2002text}. However, the RKHS space associated with these alternative kernels still possesses an inner product and norm. Thus, it is valuable to consider the geometry of the kernel space to understand neighborhood algorithms, even when the input space geometry is not explainable. To this end, we now derive properties of  NNK neighborhoods in terms of the distances and angles between the RKHS mappings of the input data.
\begin{proposition}
\label{prop:RKHS_distance}
Points in an RKHS associated with any kernel function $\kappa$ with range in $[0, 1]$ are characterized by an RKHS distance given by,
\begin{align}
    \Tilde{d}^2(i, j) = \|\vphi_i - \vphi_j\|^2 = 2 - 2\kappa(\vx_i, \vx_j). \label{eq:rkhs_dist}
\end{align}
\end{proposition}


\begin{theorem}
\label{thm: nnk_rkhs_geometry}
Given distance \plainref{eq:rkhs_dist}, the necessary and sufficient condition for two points $i,j$ to be NNK neighbors to a query $q$ is
\begin{align*}
    \theta_{i} \neq 0 \iff & \Tilde{d}^2(q,j) + \Tilde{d}^2(i,j) - \Tilde{d}^2(q,i) > \frac{\Tilde{d}^2(q,j)\Tilde{d}^2(i,j)}{2} \\ 
    \theta_{j} \neq 0 \iff & \Tilde{d}^2(q,i) + \Tilde{d}^2(i, j) - \Tilde{d}^2(q,j)  >  \frac{\Tilde{d}^2(q, i)\Tilde{d}^2(i, j)}{2} \\ 
\end{align*}
\end{theorem}
\begin{corollary}
\label{corollary:rkhs_angle}
Using the law of cosines,
\begin{align}
    \theta_{i} \neq 0 \iff \cos\alpha > \frac{\Tilde{d}(q,j)\Tilde{d}(i, j)}{4} \label{eq:rkhs_angle_i} \\
    \theta_{j} \neq 0 \iff \cos\beta > \frac{\Tilde{d}(q, i)\Tilde{d}(i, j)}{4} \label{eq:rkhs_angle_j} 
\end{align}
where $\alpha, \beta$ are the angles subtended by the chords joining ($\phi_i, \phi_j$) and ($\phi_q, \phi_j$), ($\phi_q, \phi_i$) respectively as in \Figref{fig:rkhs_geometry}.
\end{corollary}

Corollary \ref{corollary:rkhs_angle} allows us to provide a bound for the angle subtended and the length of chords for an NNK neighbor, namely, 
\begin{align}
  \theta_{i} \neq 0 \iff \frac{\pi}{3} < \alpha < - \frac{\pi}{3} \label{eq:rkhs_angle_bound}
\end{align}
where we use the fact that the RKHS distance \plainref{eq:rkhs_dist} is is upper bounded by $2$. A similar result holds for $j$ and $\beta$.
\begin{figure}[htbp]
\centering
\begin{subfigure}{0.23\textwidth}
\centering
\includegraphics[height=0.14\textheight]{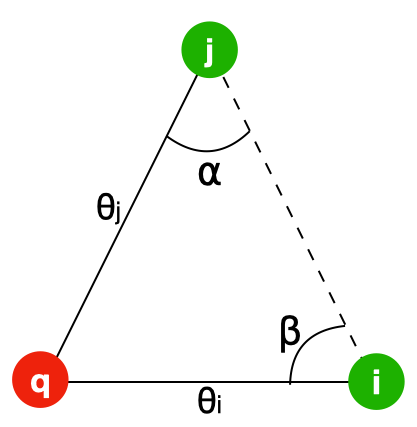}
\caption{}
\label{fig:rkhs_geometry}
\end{subfigure}
~
\begin{subfigure}{0.23\textwidth}
\centering
\includegraphics[height=0.14\textheight]{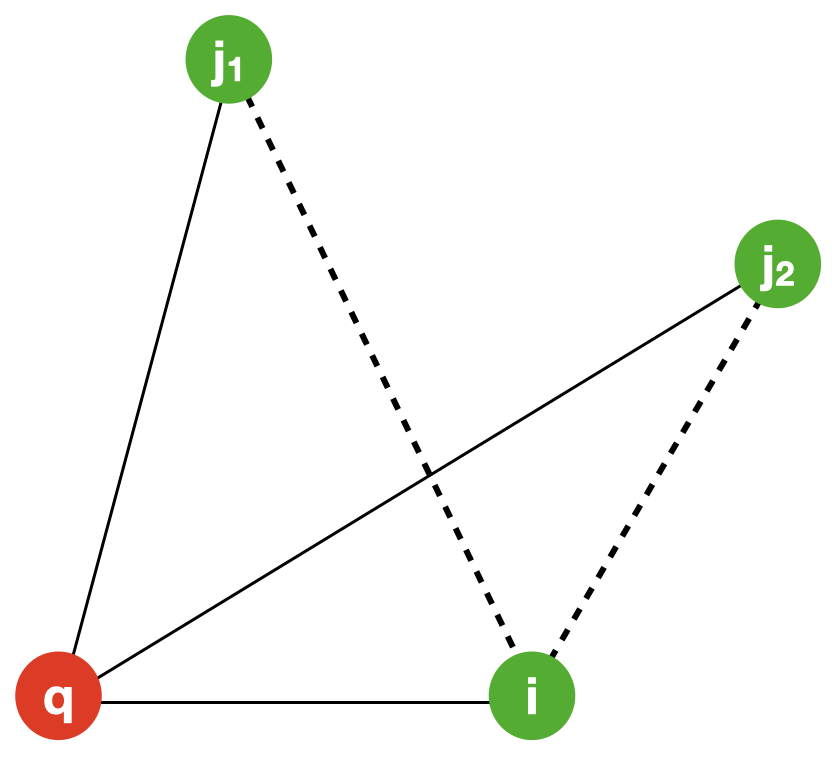}
\caption{}
\label{fig:pursuit_2nd_step_example}
\end{subfigure}
\caption{(a) Geometric setup of NNK in RKHS. (b) Problem setup at the $2^{nd}$ step of pursuit in MP, OMP neighborhoods.}
\end{figure}
\subsection{Geometry of MP, OMP neighborhoods}
\label{subsec:basis_pusuit_geometry}
We now show that Theorem \ref{th:KRI} can also be  applied to matching pursuits-based neighborhood definitions. We note that the results in this section are applicable to other problems with non-negative basis pursuits \cite{hoyer2002non, NguyenNNOG2019}. The geometric conditions presented here provide a potential approach to efficiently screen and select atoms in iterative basis pursuits for non-negative approximations with large dictionaries.  

Assume atom $\vphi_i$ is selected in the first step of either the MP or the OMP algorithm (\autoref{tab:nnk_mp_omp_algorithm}). Then, the approximation residue is given by $\vphi_{r_1} = \vphi_q - \mK_{q,i}\vphi_i$.

\noindent Now, observe that
\begin{align}
    \vphi^\top_{r_1}\vphi_q &= 1 - \mK^2_{q,i} \geq 0 \quad \text{and} \label{eq:non_zero_residue_step1}\\
    \vphi^\top_{r_1}\vphi_i &= \mK_{q,i} - \mK_{q,i}\mK_{i,i} = 0, \quad \text{since } \mK_{i,i}=1.  \label{eq:res_orthogonal_step1}
\end{align} 
Equation (\ref{eq:non_zero_residue_step1}) shows that the residue at the end of the first step can only be exactly zero iff $\mK_{q,i}=1$. 
However, the residue at the end of the first step is orthogonal to the selected atom ($\vphi_i$) as shown in \plainref{eq:res_orthogonal_step1}. 

For the second step of the algorithm, consider two points $j_1$ and $j_2$ as in Fig. \ref{fig:pursuit_2nd_step_example}. By \autoref{th:KRI}, we expect $j_2$ to be not selected and $j_1$ to be selected. The equations below show that the pursuit algorithms carry out a selection process where points that do not satisfy the KRI conditions are indeed not selected for representation.
\begin{align*}
    \vphi^\top_{r_1}\vphi_{j_2} = \mK_{q,j_2} - \mK_{q,i}\mK_{i,j_2} &\leq 0 &\iff \frac{\mK_{q,i}}{\mK_{q,j_1}} \geq \frac{1}{\mK_{i,j_1}} \\
    \vphi^\top_{r_1}\vphi_{j_1} = \mK_{q,j_1} - \mK_{q,i}\mK_{i,j_1} &> 0 &\iff \frac{\mK_{q,i}}{\mK_{q,j_1}} < \frac{1}{\mK_{i,j_1}}
\end{align*}
The weights after selection for the set $S = \{i, j_1\}$ in the case of MP is given by  $\vtheta^{(MP)}_{S} = \left[\mK_{q,i} \; \mK_{q,j_1} - \mK_{q,i}\mK_{i,j_1} \right]^\top$, while in OMP, $\vtheta^{(OMP)}_S$ is computed by minimizing (\ref{eq:neighborhood_nnk_objective}).
In MP, the updated residue is orthogonal only to the last selected data point $j_1$,
\begin{align}
    \vphi^{(MP)\top}_{r_2}\vphi_i =& \mK_{q,i} - \mK_{q,i}\mK_{i,i} -(\mK_{q,j_1} - \mK_{q,j_1}\mK_{i,j_1})\mK_{i,j_1} \nonumber\\
    =& -(\mK_{q,j_1} - \mK_{q,j_1}\mK_{i,j_1})\mK_{i,j_1} \neq 0 \label{eq:res_non_orthogonal_step2}\\
    \vphi^{(MP)\top}_{r_2}\vphi_{j_1} =& \mK_{q,j_1} - \mK_{q,i}\mK_{i,j_1} -(\mK_{q,j_1} - \mK_{q,j_1}\mK_{i,j_1}) \nonumber \\ 
    =& 0 \label{eq:res_orthogonal_step2}
\end{align}
while in OMP, the residue is orthogonal to all selected atoms as guaranteed by the first-order optimality condition of the objective (\ref{eq:neighborhood_nnk_objective}), i.e, $\mK_{S,S}\vtheta^{(OMP)} - \mK_{S,q} = 0$.
\begin{align}
    \vphi^{(OMP)\top}_{r_2}\vphi_i =& \mK_{q,i} - \vtheta^{(OMP)\top}_S\mK_{S,i} = 0\\
    \vphi^{(OMP)\top}_{r_2}\vphi_{j_1} =& \mK_{q,j_1} - \vtheta^{(OMP)\top}_S\mK_{S,j_1} = 0
\end{align}
This scenario repeats in the subsequent pursuit steps where selected neighbors in MP are orthogonal only to the last neighbor selected, while in OMP to all selected neighbors. The iterative selection ends when the maximum number of allowed neighbors is reached or when no positively correlated atom is available for representing the residue i.e,
\begin{align*}
    \forall m \notin S \quad \mK_{q,m} - \vtheta_S\mK_{S,m} \leq 0
\end{align*}
The above criteria prevent pursuit algorithms from selecting points that do not contribute to the non-negative approximation as these correspond to points outside the polytope (Corollary \ref{corollary:polytope_optimality}). 
Note that, unlike unconstrained pursuit where adding more atoms often leads to a better approximation, increasing the number of selected atoms does not necessarily correspond to better representation in non-negative pursuit \cite{NguyenNNOG2019}.

\section{Experiments and Results}
\label{sec:experiments}
We demonstrate the benefits of the proposed NNK framework in neighborhood-based classification, graph-based label propagation and dimensionality reduction. 
Code for experiments is available at  \url{https://github.com/STAC-USC/}. 
\subsection{NNK neighborhood}
\label{sec:nnk_neighbor_experiments}
We study the proposed NNK neighborhood for classification~(binary and multi-class) using a plug-in classifier $\sI(\hat{f}(\vx_q) > 0)$ based on the neighborhood estimate 
\begin{align}
\hat{f}(\vx_q) = \sum_{i \in \gN(\vx_q)}\frac{\theta_{i}}{\sum_{j \in \gN(\vx_q)} \theta_j}\;y_i, \label{eq:local_estimate}
\end{align}
where $\gN(\vx_q)$ is the neighborhood defined for query $\vx_q$ and $\theta_i, y_i$ are the vector of weights and the one-hot encoded label associated with $\vx_i$, respectively.

\textbf{Experiment setup:}
We compare the NNK neighborhood against two baselines, the standard weighted $\rk$NN and an adaptive neighborhood approach $\rk^*$NN \cite{anava2016k}.
We make use of the Gaussian kernel~\plainref{eq:gaussian_kernel} for neighborhoods defined with $\rk$NN and NNK. 
Two groups of datasets are considered: $(i)$ AR face \cite{martinez1998ar}, Extended YaleB \cite{georghiades2001few}, Isolet, and USPS \cite{gadde2014active} with their standard train/test split, and $(ii)$ $70$ datasets from OpenML \cite{rijn2013openml}. The OpenML datasets are obtained based on the number of dimensions ($d \in [20, 1000]$), the number of samples ($N \in [1000, 20000]$),  with no restriction on the number of classes as in \cite{ram2022federated}.  
All datasets are standardized using the empirical mean and variance estimated with training data. We repeat all experiments $10$ times and report average performances and ablation studies. 

\begin{figure}[htbp]
\begin{subfigure}{0.125\textwidth}
    \includegraphics[trim={1cm 6.5cm 2cm 7cm},clip,width=\textwidth]{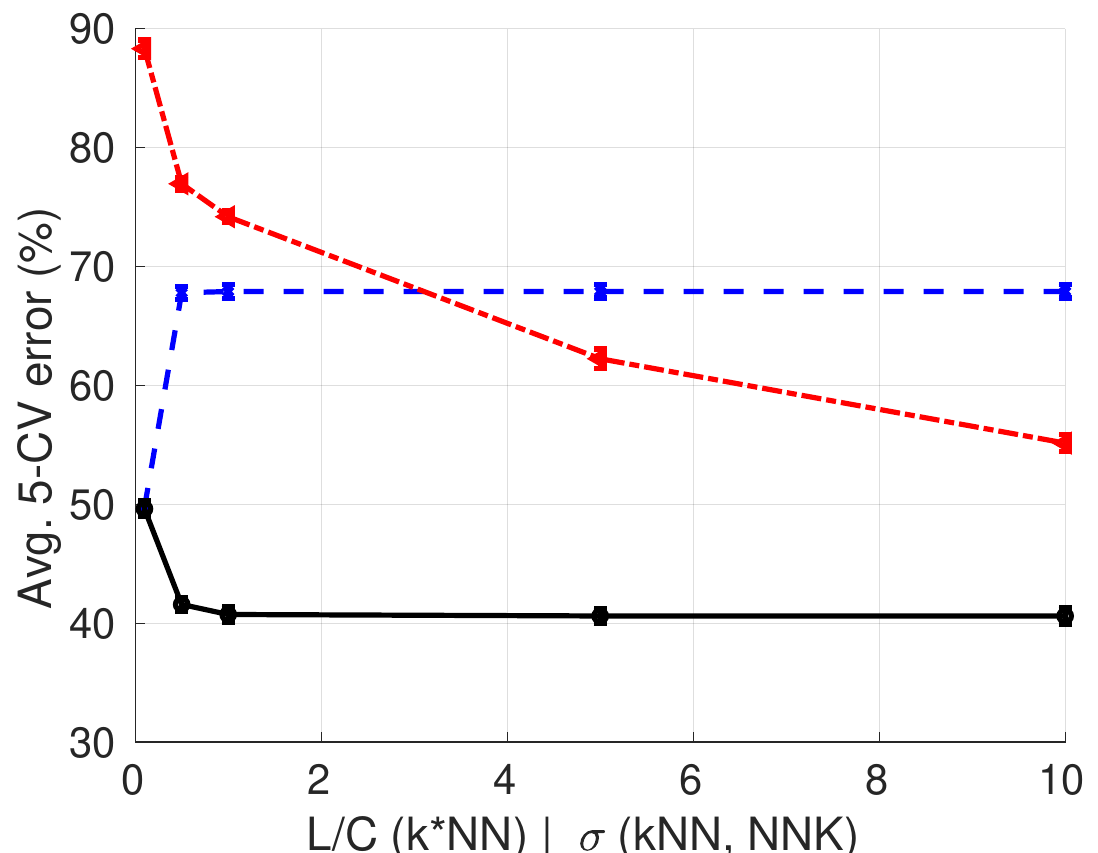}
\end{subfigure}
\begin{subfigure}{0.115\textwidth}
    \includegraphics[trim={2.2cm 6.5cm 2cm 7cm},clip,width=\textwidth]{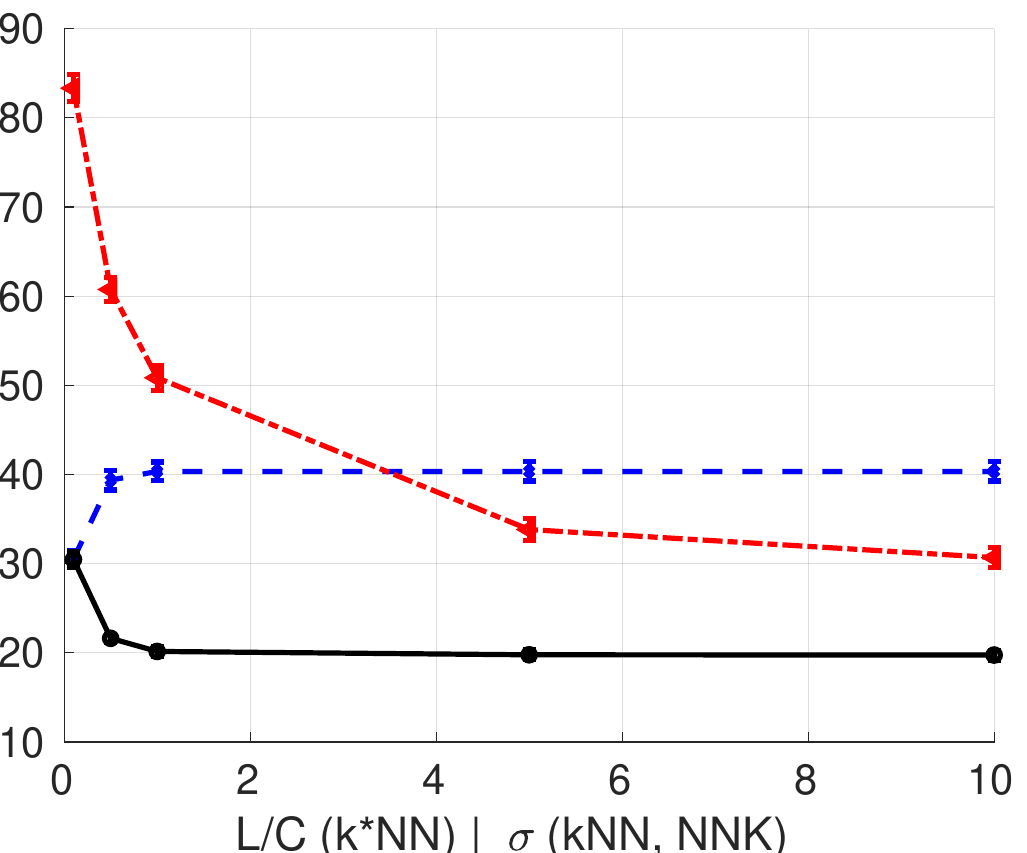}
\end{subfigure}
\begin{subfigure}{0.115\textwidth}
    \includegraphics[trim={2.2cm 6.5cm 2cm 7cm},clip,width=\textwidth]{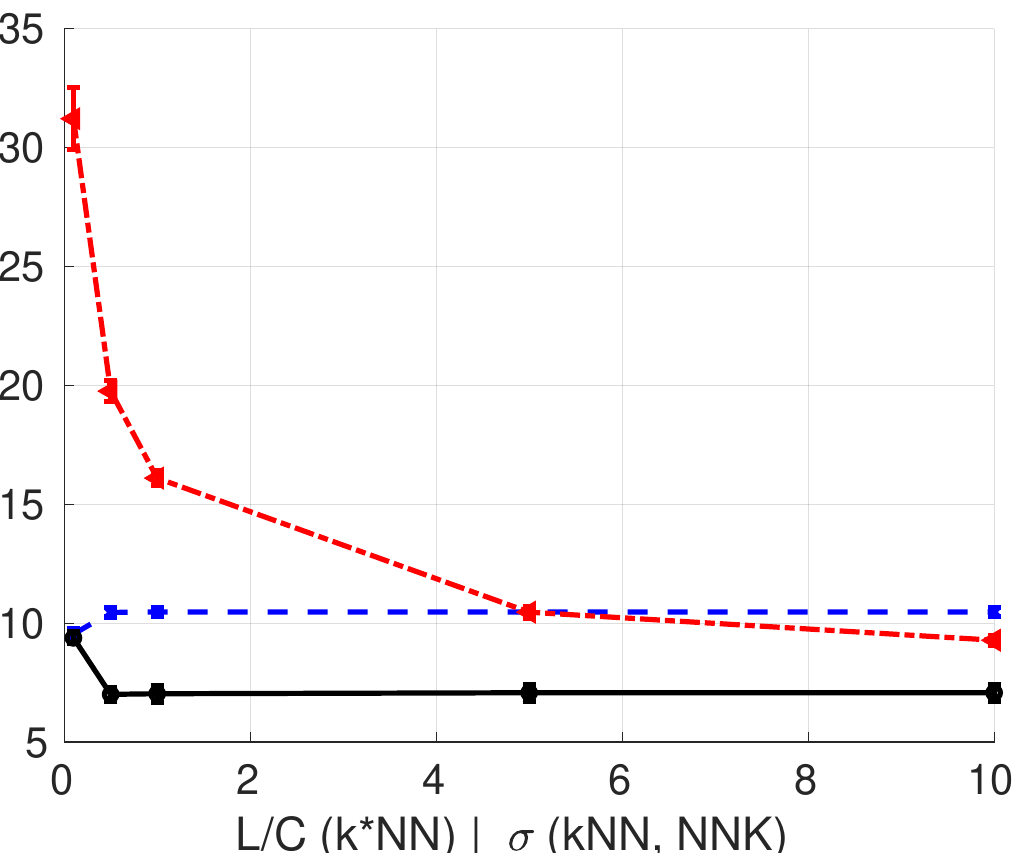}
\end{subfigure}
\begin{subfigure}{0.115\textwidth}
    \includegraphics[trim={2.2cm 6.5cm 2cm 7cm},clip,width=\textwidth]{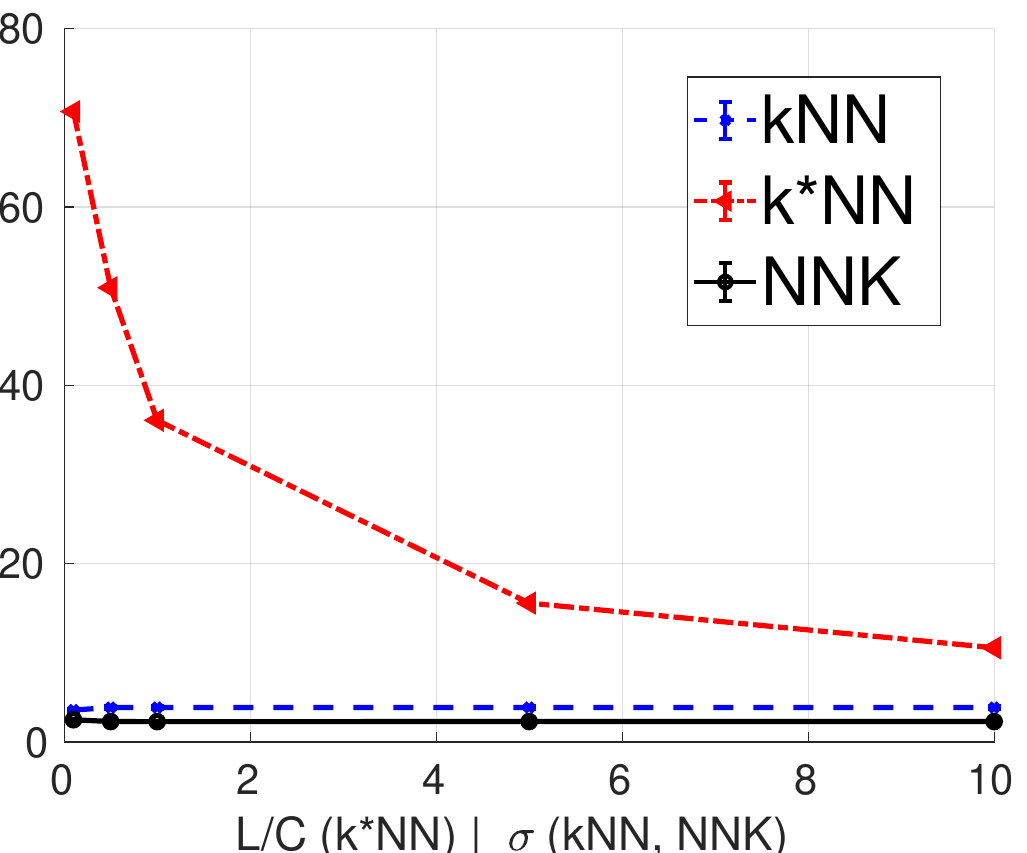}
\end{subfigure}

\begin{subfigure}{0.125\textwidth}
    \includegraphics[trim={1cm 6.5cm 2cm 7cm},clip,width=\textwidth]{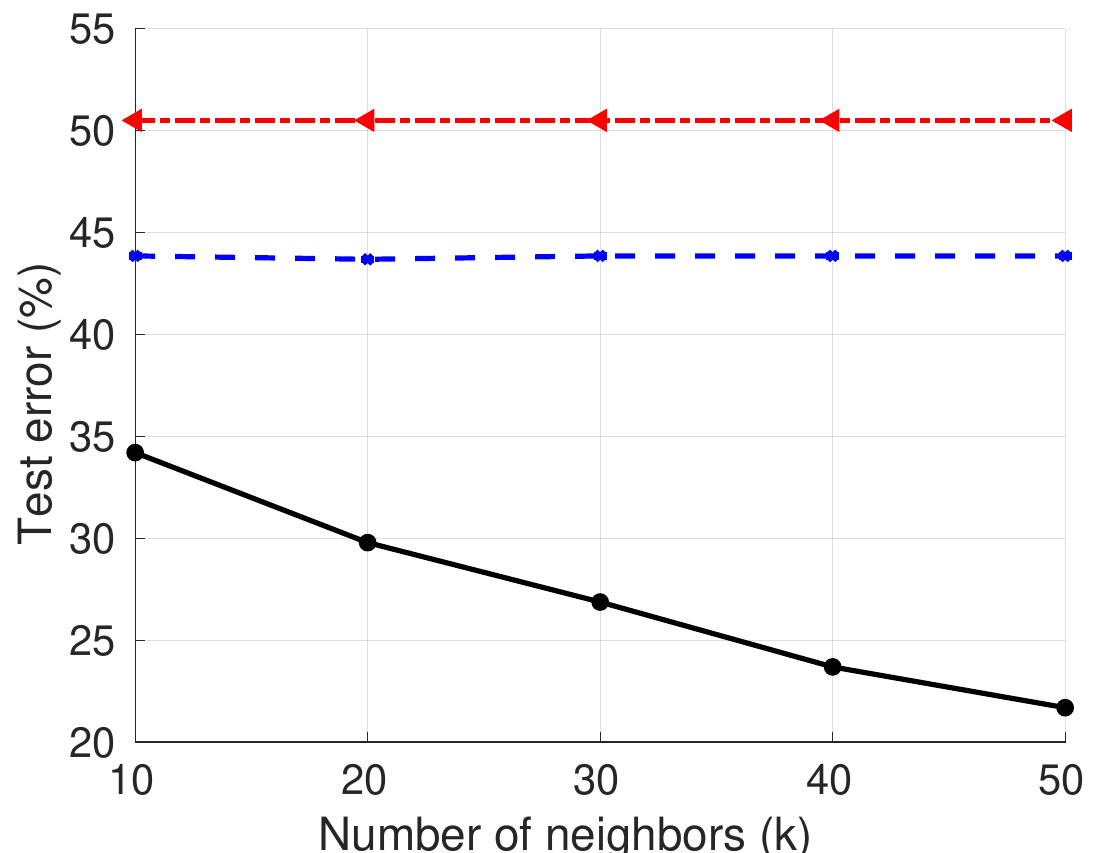}
    \caption{AR Face}
\end{subfigure}
\begin{subfigure}{0.115\textwidth}
    \includegraphics[trim={2.2cm 6.5cm 2cm 7cm},clip,width=\textwidth]{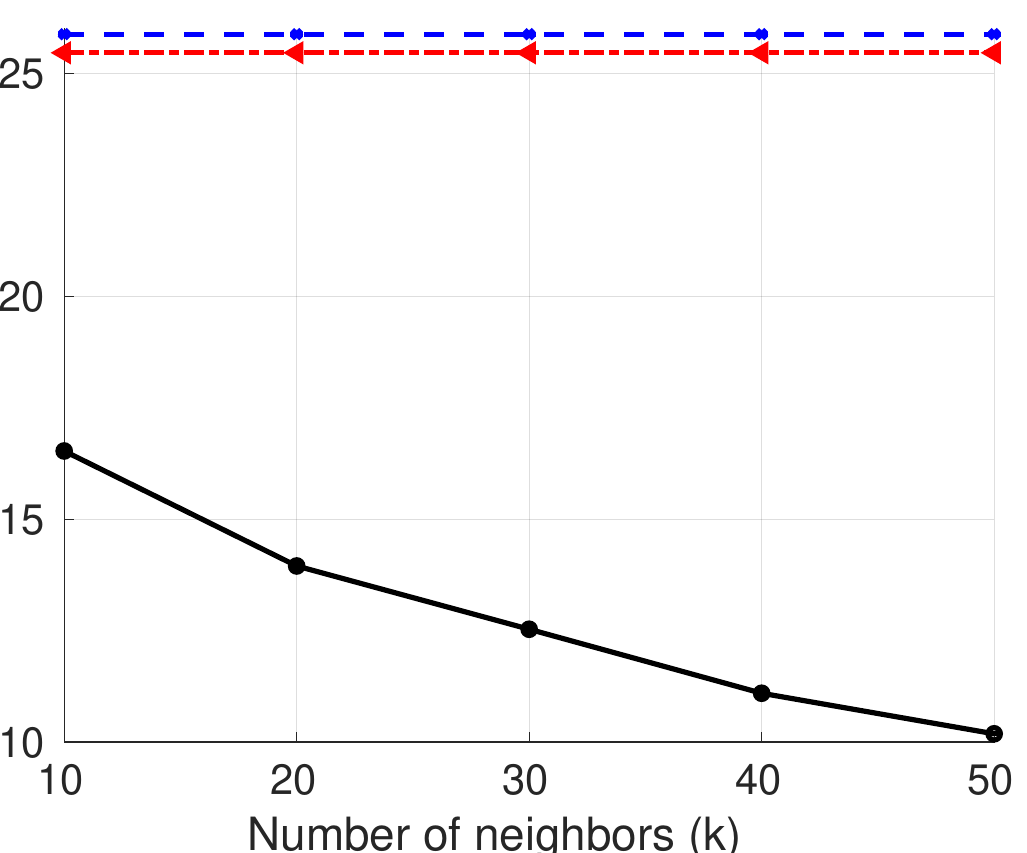}
    \caption{YaleB Face}
\end{subfigure}
\begin{subfigure}{0.115\textwidth}
    \includegraphics[trim={2.0cm 6.5cm 2cm 7cm},clip,width=\textwidth]{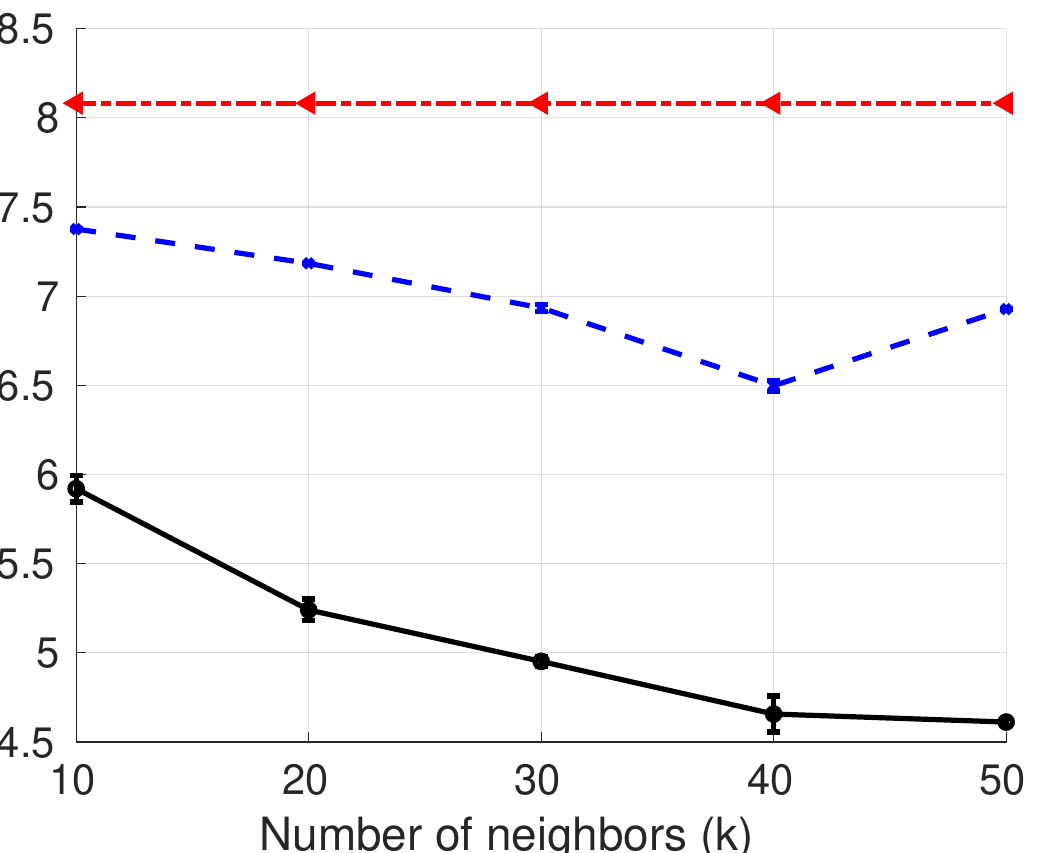}
    \caption{Isolet}
\end{subfigure}
\begin{subfigure}{0.115\textwidth}
    \includegraphics[trim={2.2cm 6.5cm 2cm 7cm},clip,width=\textwidth]{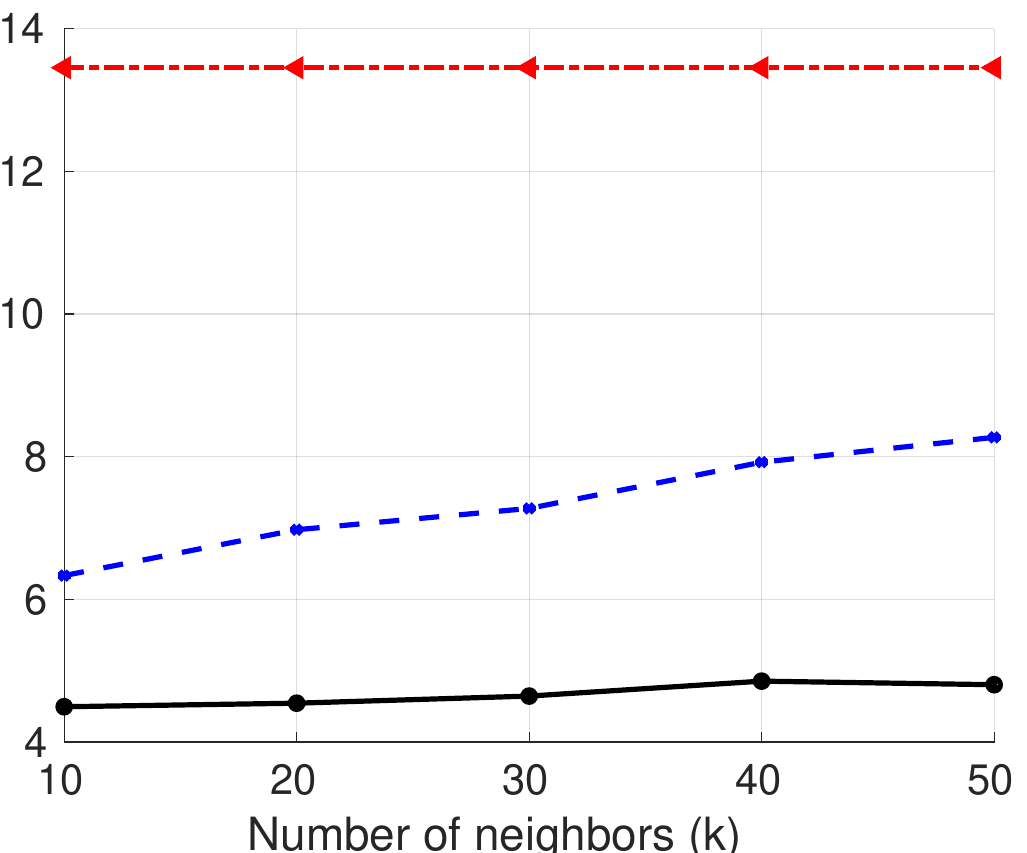}
    \caption{USPS}
\end{subfigure}
\caption{\textbf{Top:} Average cross-validation error for different values of hyperparameters associated with $\rk^*$NN , $\rk$NN, and NNK. The value of $\rk$ in $\rk$NN and NNK is set to $10$ in these experiments. \textbf{Bottom:} Test error of $\rk$NN and NNK for different values of $\rk$ with hyperparameter $\sigma$ selected using cross-validation. Here we include $\rk^*$NN for reference. 
We see that NNK, unlike $\rk$NN and $\rk^*$NN,  produces better performance for a wide range of hyperparameter values.
}
\label{fig:ablation_study}
\end{figure}

\textbf{AR, YaleB, Isolet, USPS:} We use these datasets to perform ablation and $5$-fold cross-validation studies on parameters $\rk, \sigma$ used in $\rk$NN and NNK, and $L/C$ value in $\rk^*$NN. We consider values  $\rk\in \{10, 20, 30, 40, 50\}$ and $\sigma, L/C$ in $\{ 0.1, 0.5, 1, 5, 10\}$ as in \cite{anava2016k}.
\begin{table}[htbp]
    \centering
    \begin{tabular}{l c c c c c c}
    \toprule
        Dataset & $N_{tr}, N_{te}$ & $d$ & $N_{class}$ & $\rk$NN & $\rk^*$NN & NNK \\\midrule
        AR  & $2000, 600 $ & $540$ & $ 100$ & $43.68$ & $50.5$ & $\mathbf{21.68}$ \\
        YaleB & $1216, 1198$ & $504$ & $ 38$ & $25.88$ & $25.46$ & $\mathbf{10.18}$ \\
        Isolet & $6238, 1559$ & $617$ & $ 26$ & $6.50$& $8.08$ & $\mathbf{4.61}$ \\
        USPS & $7291, 2007$ & $256$ & $ 10$ & $6.33$ & $13.45$ & $\mathbf{4.49}$ \\
        \bottomrule
    \end{tabular}
    \caption{Test classification error (in $\%$, lower is better) using local neighborhood methods on datasets with train ($N_{tr}$) and test ($N_{te}$) sample size, dimension ($d$), and classes ($N_{class}$). NNK-based classification consistently outperforms baseline methods while being robust to choices of $\rk, \sigma$.}
    \label{tab:neighborhood_estimation}
\end{table}
\Tabref{tab:neighborhood_estimation} presents the classification performance on the test datasets with different neighborhood classifiers using training data. Reported performance corresponds to the best set of hyperparameters (obtained with cross-validation) for each method and demonstrates the gains of NNK relative to the two baselines. 

\Figref{fig:ablation_study} shows the cross-validation and test classification performance for various parameter choices. We note that both $\rk$NN and NNK are less affected by the hyperparameter $\sigma$ relative to $\rk^*$NN, which is heavily influenced by the choice of $L/C$. In particular, we see that there exists a narrow choice of $\sigma$ where $\rk$NN achieves maximum performance, as larger $\sigma$ leads to higher weights assigned to farther away neighbors. In contrast, the NNK algorithm produces a robust performance for a broad range of $\sigma$.
Further, it can be seen that NNK estimates saturate with increasing $\rk$. The  stability in NNK estimation is due to the fact that each neighbor is selected only if it belongs to a new direction in space not previously represented. 
Note that NNK estimation outperforms $\rk$NN and $\rk^*$NN, in these datasets, even for a sub-optimal choice of the kernel parameter $\sigma$ and the number of neighbors $\rk$.

\textbf{OpenML datasets:} 
The train/test split for each dataset is obtained by randomly dividing the data into two equal halves with hyperparameters $\sigma$ or $L/C$ obtained via cross-validation as in the previous set of experiments. The value of $\rk$ in $\rk$NN and NNK is fixed to $30$. 
\begin{table}[htbp]
    \centering
    \begin{tabular}{l c c c}
    \toprule
        Method & Avg. error ($\%$) & Wins/Ties/Losses & Avg. Rank  \\\midrule
        $\rk$NN & 18.06 (0.83) & 7/24/39 & 2.0571 \\
        $\rk^*$NN & 20.71 (1.04) & 19/10/50 & 2.2714 \\
        NNK & \textbf{16.26 (0.74)} & \textbf{44/9/17} & \textbf{1.6714} \\\bottomrule
    \end{tabular}
    \caption{Performance summary of neighborhood methods on $70$ OpenML datasets. Listed metrics include (i) Average classification error and standard deviation (in parentheses) across all datasets, (ii) the number of datasets where a method performs better than the other two (Win), performs within one standard deviation of the best performance (Tie), and performs poorly in comparison (Loss), and (iii) average rank of the method based on the classification ($1^{st}, 2^{nd}, 3^{rd}$).}
    \label{tab:openml_performance}
\end{table}
\Tabref{tab:openml_performance} provides a summary of the classification performance on all $70$ datasets. 
We observe that NNK achieves better overall performance in comparison to baselines. $\rk^*$NN, though successful in some cases, falls short in several datasets, which is indicative of its inability to adapt to varied dataset scenarios. In contrast, NNK demonstrates flexibility to feature and class disparities in several datasets. The failure cases of NNK, such as the binarized version of  regression designed in \cite{friedman1991multivariate} ($11$ of $17$ losses),  correspond to datasets with non-informative feature dimensions containing only noise. 
In such scenarios, redundancy in the representation, as in $\rk$NN, can provide a better estimate. We leave this NNK limitation for future study.       
\subsection{NNK graphs}
\label{sec:nnk_graph_experiments}
In this section, we compare NNK graphs (\Algref{tab:nnk_graph_algorithm}) and the basis pursuit variants (MP, OMP from \Algref{tab:nnk_mp_omp_algorithm}) with graph construction methods based on weighted $\rk$NN, LLE \cite{Wang2008}, AEW \cite{Karasuyama2017}, and the method by Kalofolias \cite{kalofolias2018large}. 
We measure the goodness of the graphs in terms of the sparsity of the graphs, runtime required for graph construction, 
and performance on downstream tasks. 
As in neighborhood experiments, we use the Gaussian kernel defined in \plainref{eq:gaussian_kernel}. We set $\sigma$ parameter such that the $\rk$NN approach is able to assign a nonzero weight to all its $\rk$ neighbors, i.e. a connected graph, by ensuring the $\rk^{th}$ neighbor distance is within $3*\sigma$. 
For a fair comparison, we use the same $\sigma$ value for  AEW, MP, OMP, and  NNK graph constructions.

\begin{figure}[tbp]
\centering
\begin{subfigure}{0.15\textwidth}
\includegraphics[trim={2cm 1cm 2cm 1cm},clip, width=\textwidth]{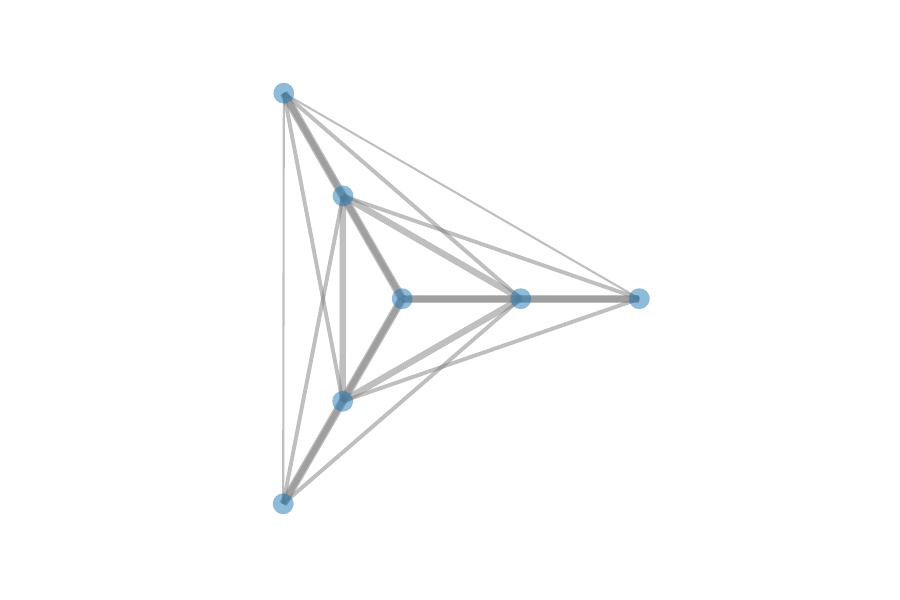}
\end{subfigure}
\begin{subfigure}{0.15\textwidth}
\includegraphics[trim={2cm 1cm 2cm 1cm},clip, width=\textwidth]{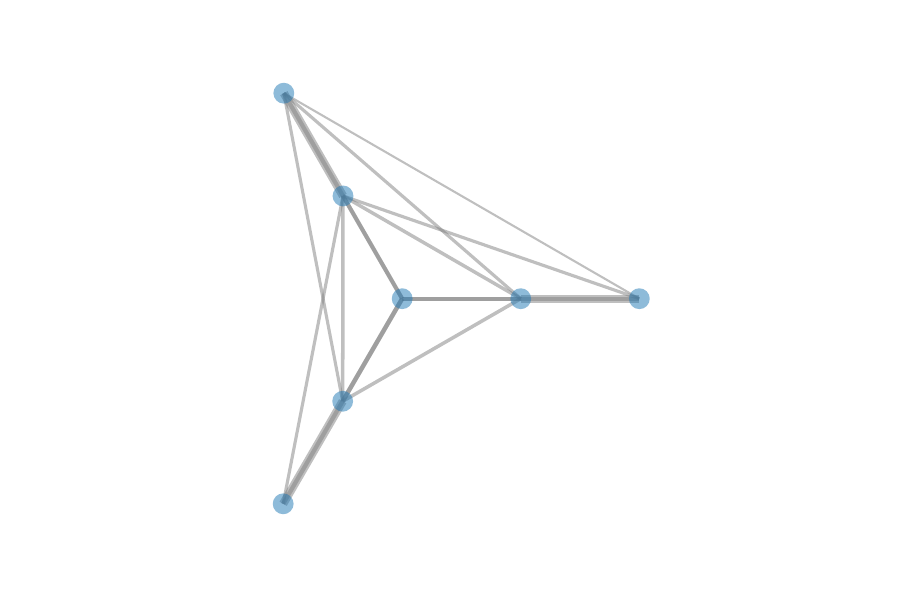}
\end{subfigure}
\begin{subfigure}{0.15\textwidth}
\includegraphics[trim={2cm 1cm 2cm 1cm},clip, width=\textwidth]{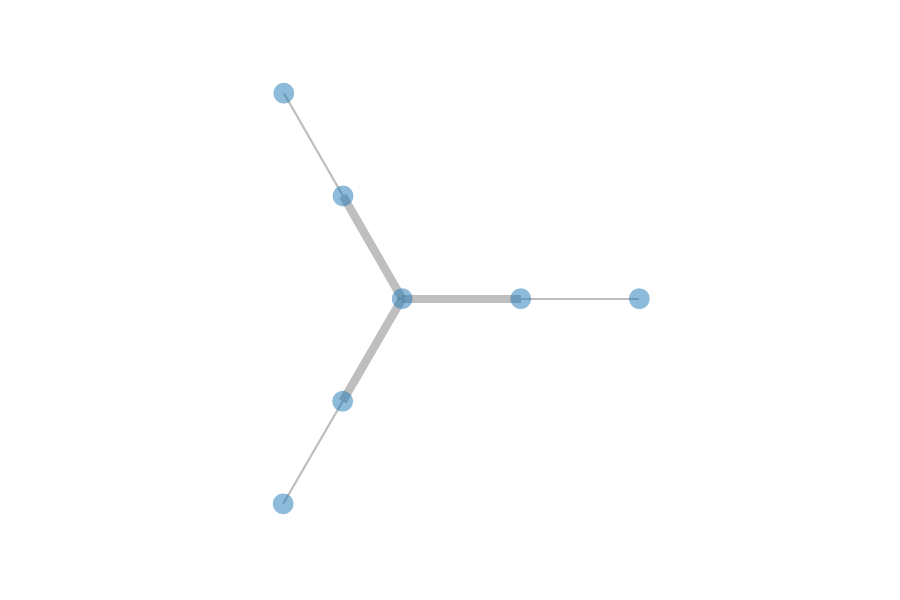}
\end{subfigure}

\begin{subfigure}{0.15\textwidth}
\includegraphics[trim={2cm 1cm 2cm 1cm},clip, width=\textwidth]{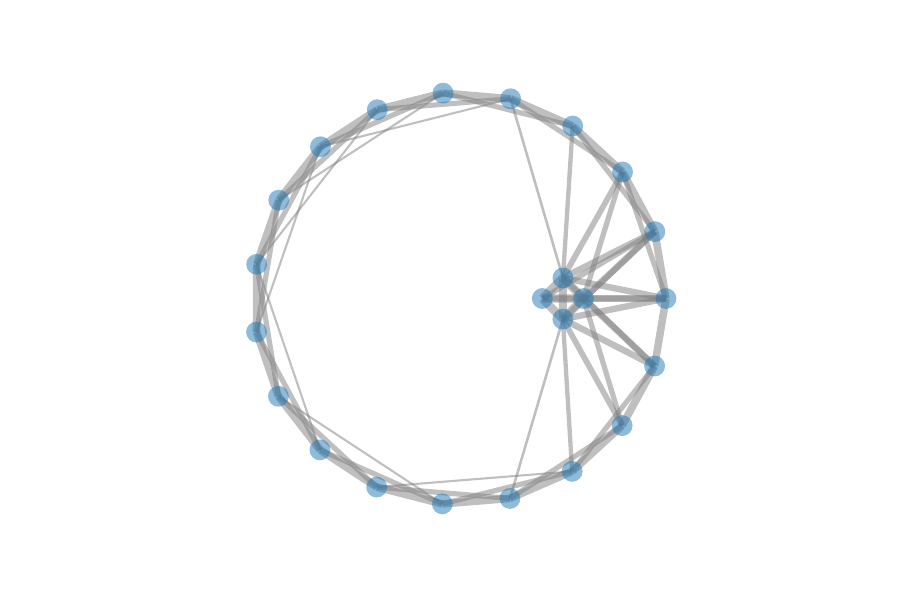}
\caption{$\rk$NN Graph}
\end{subfigure}
\begin{subfigure}{0.15\textwidth}
\includegraphics[trim={2cm 1cm 2cm 1cm},clip, width=\textwidth]{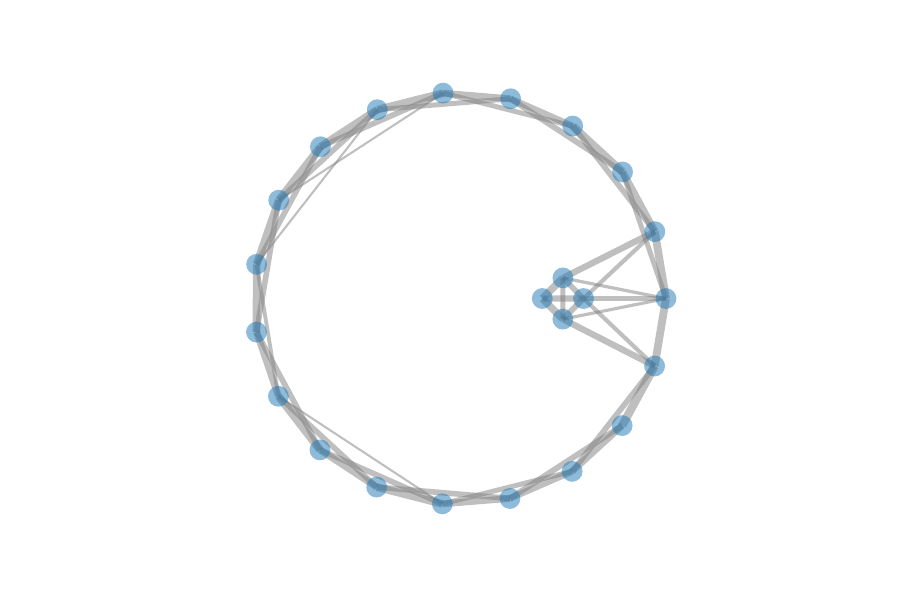}
\caption{LLE Graph}
\end{subfigure}
\begin{subfigure}{0.15\textwidth}
\includegraphics[trim={2cm 1cm 2cm 1cm},clip, width=\textwidth]{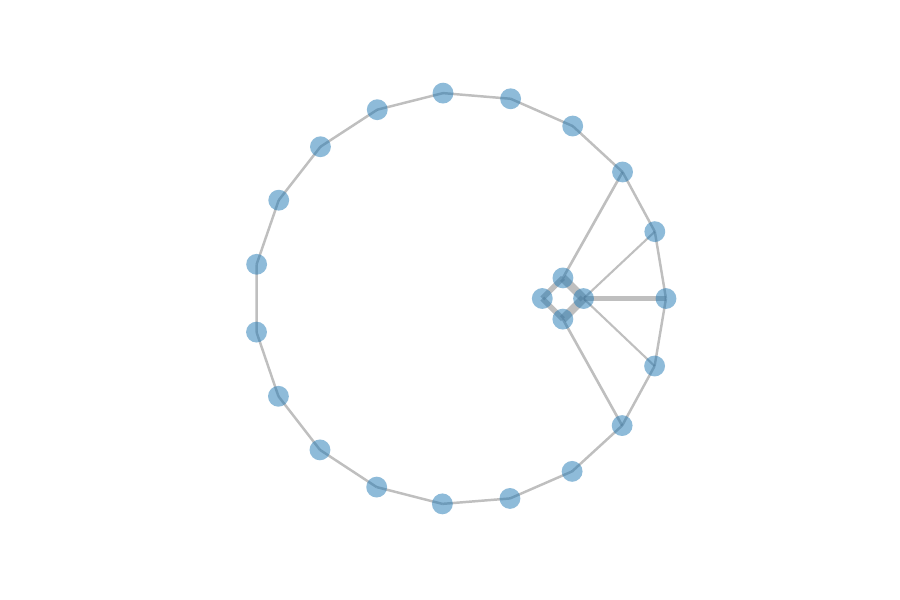}
\caption{NNK Graph}
\end{subfigure}
\caption{Graph constructions in synthetic data settings. Unlike $\rk$NN and LLE, NNK achieves a sparse representation where the \emph{relative locations} of nodes in the neighborhood are key. In the seven-node example (top row), the node in the center has only three direct neighbors, which in turn are neighbors of nodes extending the graph further in those three directions. Graphs are constructed with $\rk=5$ and Gaussian kernel~($\sigma=1$). Edge thickness indicates neighbor weights normalized globally to have values in $[0, 1]$.}
\label{fig:2d_example}
\end{figure}
\begin{figure}[htbp]
    \centering
    \begin{subfigure}{0.24\textwidth}
        \includegraphics[trim={1.2cm 6cm 1.5cm 6.7cm},clip,width=\textwidth]{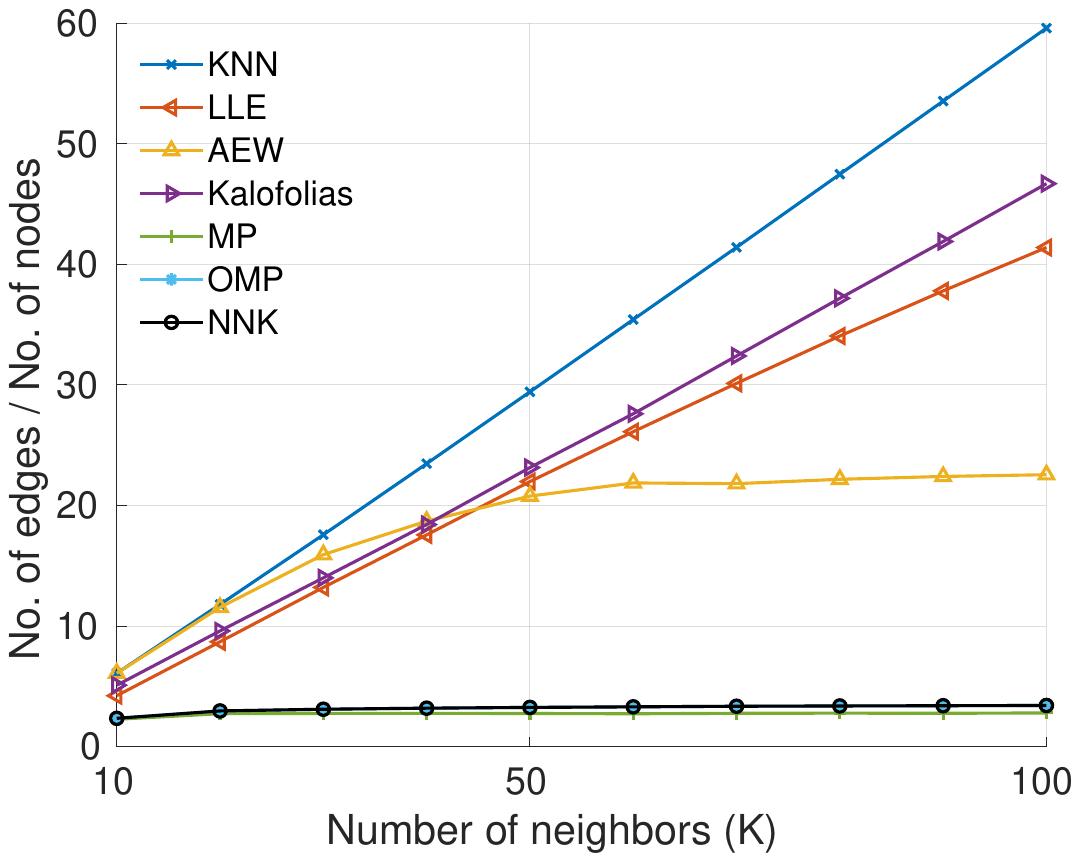}
    \end{subfigure}
    \begin{subfigure}{0.24\textwidth}
        \includegraphics[trim={1.2cm 6cm 1.5cm 6.7cm},clip,width=\textwidth]{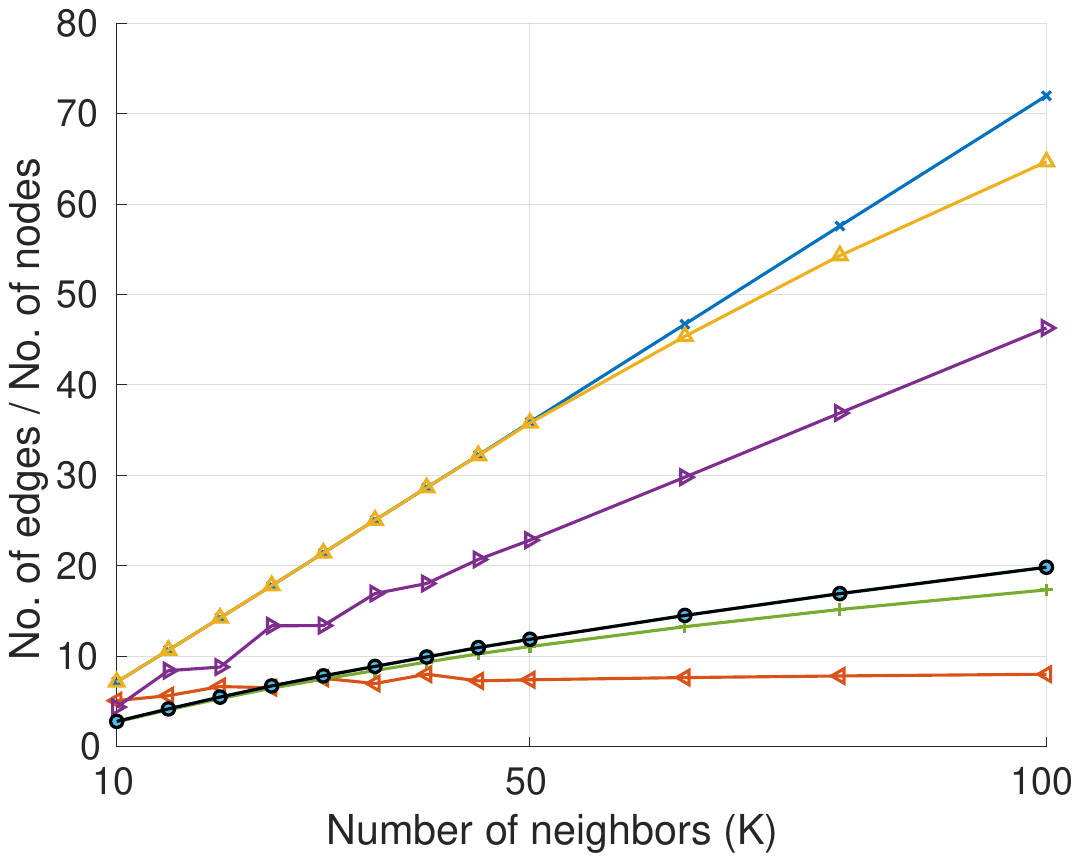}
    \end{subfigure}
    \caption{Ratio of the number of edges (with weights $\geq 10^{-8}$) to the number of nodes.  \textbf{Left:} Noisy Swiss Roll (N=5000, d=3),  \textbf{Right:} MNIST (N=1000, d=784). The sparsity of NNK, MP, OMP graphs saturates for increasing $\rk$ and is reflective of the intrinsic dimensionality of the data.}
    \label{fig:swiss_roll_dimensionality}
\end{figure}
\subsubsection{Sparsity and runtime complexity} 
\label{sec:sparsity_runtime_graph}
Graph sparsity, measured as the ratio of the number of edges to the number of nodes, is desirable for 
downstream graph analysis tasks to scale well for large data sets \cite{Batson2013}. 
Unlike a $\rk$NN graph, where sparsity is controlled by the choice of $\rk$, an NNK graph provides an adaptive sparse representation. \Figref{fig:2d_example} presents two examples to compare the NNK graph  sparsity with that of the $\rk$NN and LLE graphs.
In \Figref{fig:swiss_roll_dimensionality}, we show the effect of parameter $\rk$ on the sparsity of various graph methods on a synthetic and real dataset. We see that the MP, OMP, and NNK algorithms lead to the sparsest graph representations.
Note that LLE graphs suffer from unstable optimization due to the high dimensionality of the features in the real dataset. Thus, the observed sparsity, in this case, is not reflective of a good graph and is only a shortcoming of the optimization (weights $<10^{-8}$).

As shown in  \Figref{fig:ssl_performance}, the NNK framework outperforms other methods in terms of both sparsity and the time required to construct the graphs. An OMP graph construction produces an equivalent graph to NNK at a $10\times$ slower runtime. MP method provides a faster alternative to OMP but still remains slow with runtime $2.5\times$ that of NNK. 


\subsubsection{Label Propagation}
\label{sec:graph_ssl}
We evaluate the quality of the constructed graphs based on their performance in a downstream semi-supervised learning~(SSL) task using label propagation \cite{Zhu-SSL-ICML-03}. 
We consider subsets of digits datasets, such as USPS and MNIST, for the experiment. For USPS, we sample each class non-uniformly based on its labels 
as in \cite{Zhu-SSL-ICML-03,kalofolias2016learn}. For MNIST, we use $100$ randomly selected samples for each digit class to create one instance of a $1000$ sample dataset~($100$ samples/class $\times$ $10$ digit classes). 
Fig.~\ref{fig:ssl_performance} shows the average classification error with different graph constructions for an increasing percentage of available labels (randomly selected) and different choices of parameter $\rk$. We see that basis pursuit and NNK  perform best with both combinatorial ($\mL$) and symmetric normalized Laplacians ($\gL$) in label propagation for all settings. In particular, we note that NNK and OMP-based constructions result in the same graph structure and performance, as identified in our analysis in \Twosecrefs{sec:mp_omp_neighborhoods}{subsec:basis_pusuit_geometry}, but with a significant difference in runtimes (NNK approach requires much less time compared to OMP).  

\begin{figure}[tbp]
\centering
\begin{subfigure}{0.22\textwidth}
\includegraphics[trim={1cm 6cm 1.5cm 7cm},clip,width=\textwidth]{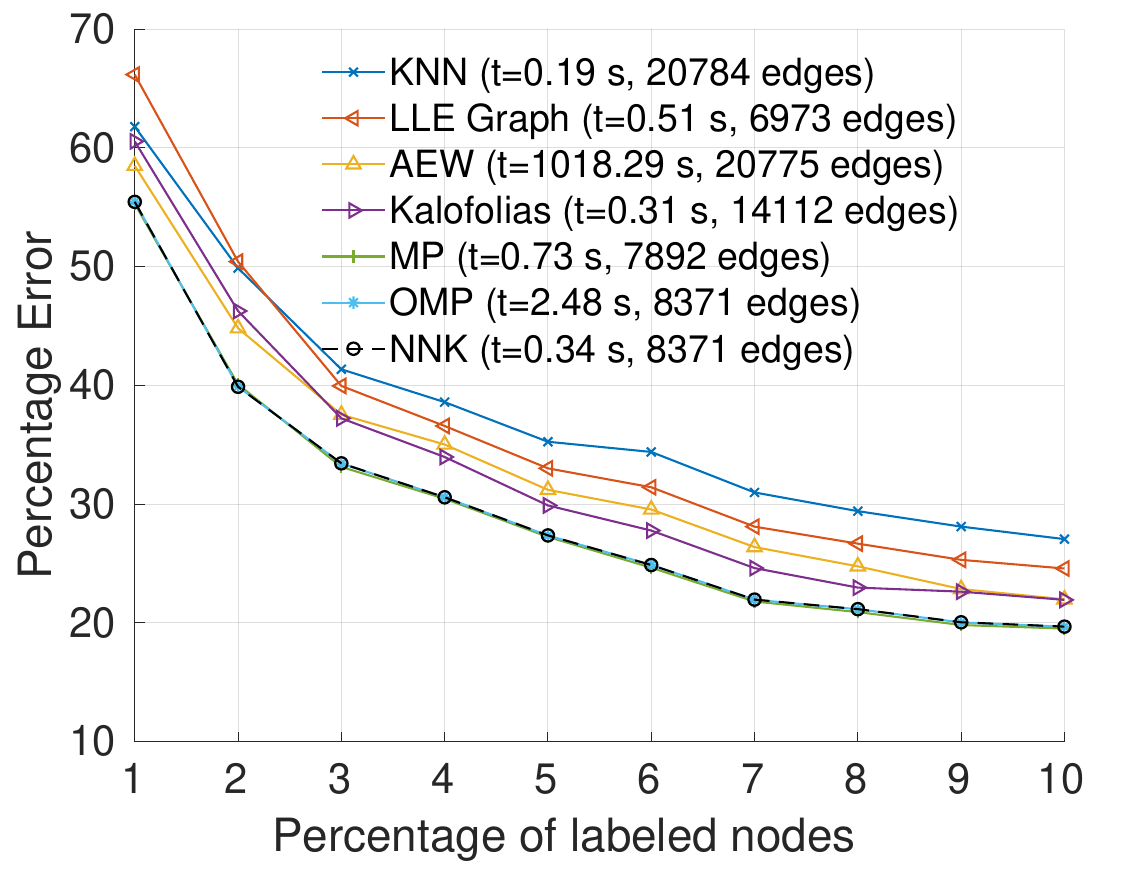}
\end{subfigure}
~
\begin{subfigure}{0.22\textwidth}
\includegraphics[trim={1.5cm 6cm 1.5cm 7cm},clip,width=\textwidth]{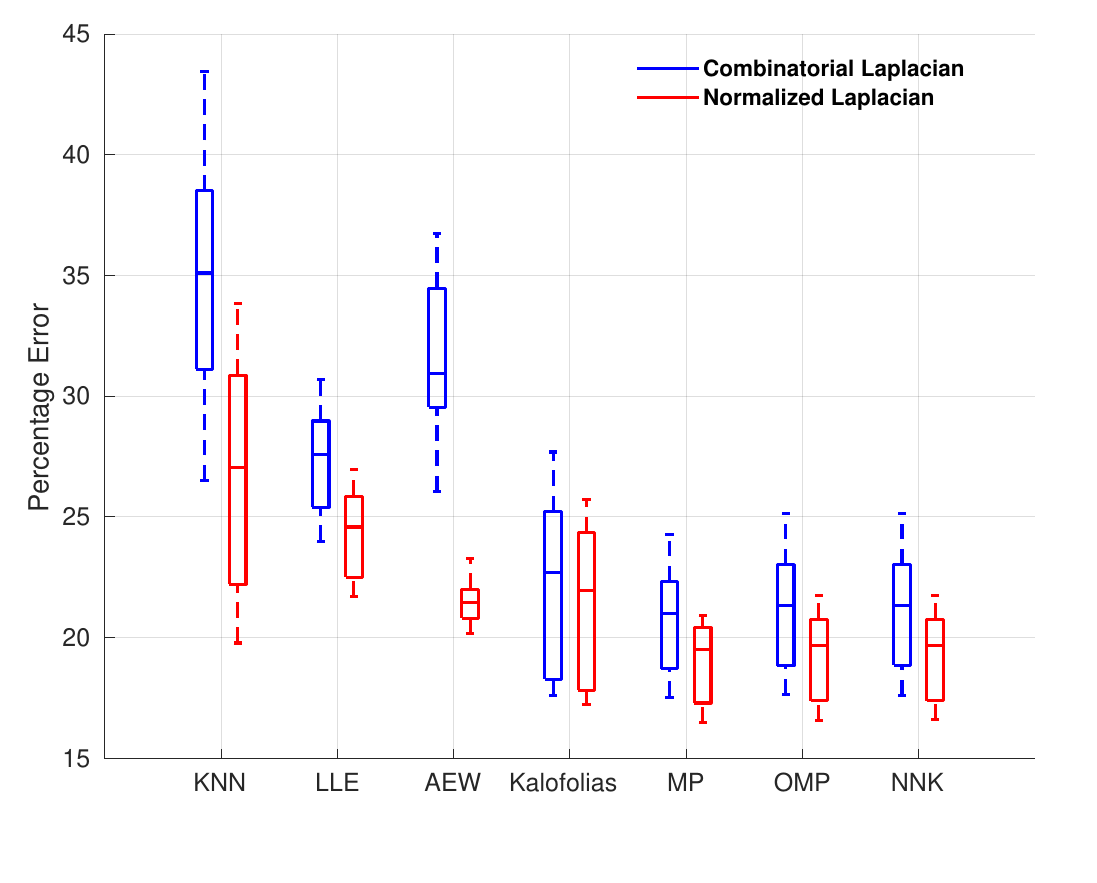}
\end{subfigure}

\begin{subfigure}{0.22\textwidth}
\includegraphics[trim={1cm 6cm 1.5cm 7cm},clip,width=\textwidth]{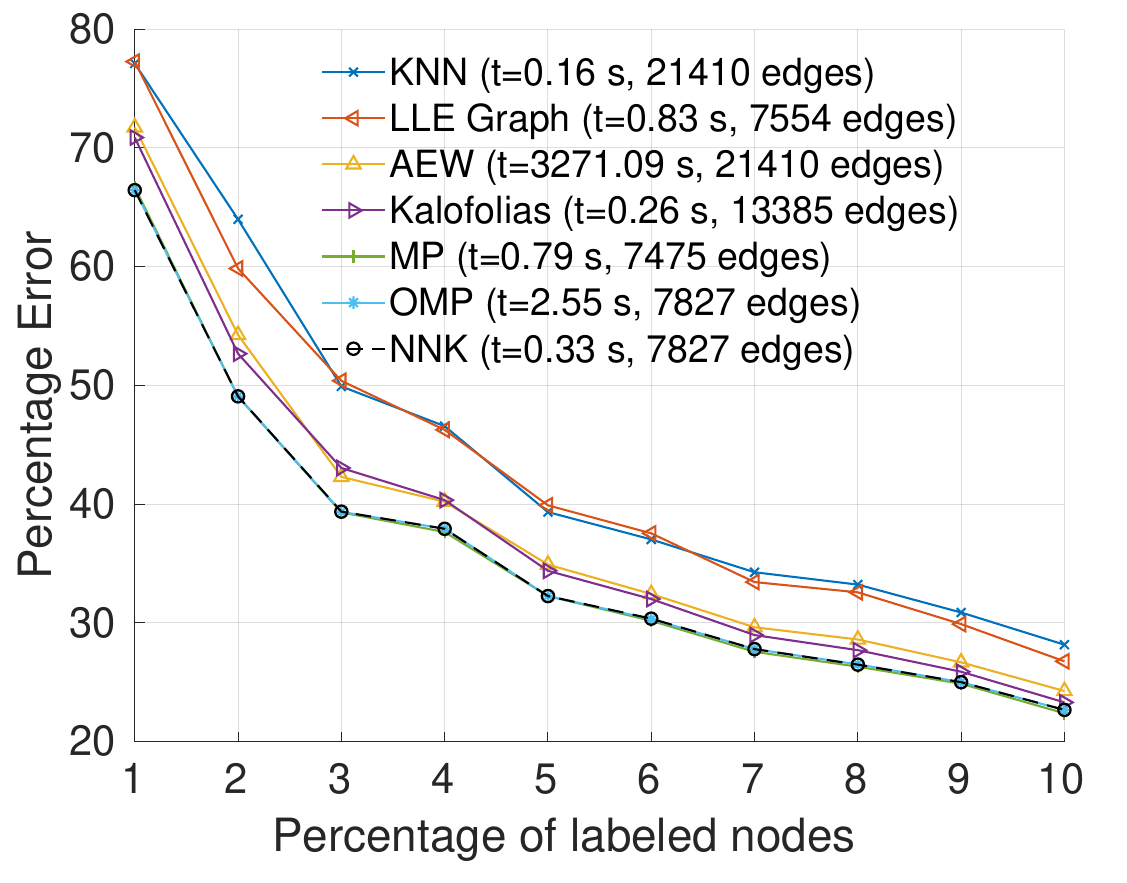}
\end{subfigure}
~
\begin{subfigure}{0.22\textwidth}
\includegraphics[trim={1.5cm 6cm 1.5cm 7cm},clip,width=\textwidth]{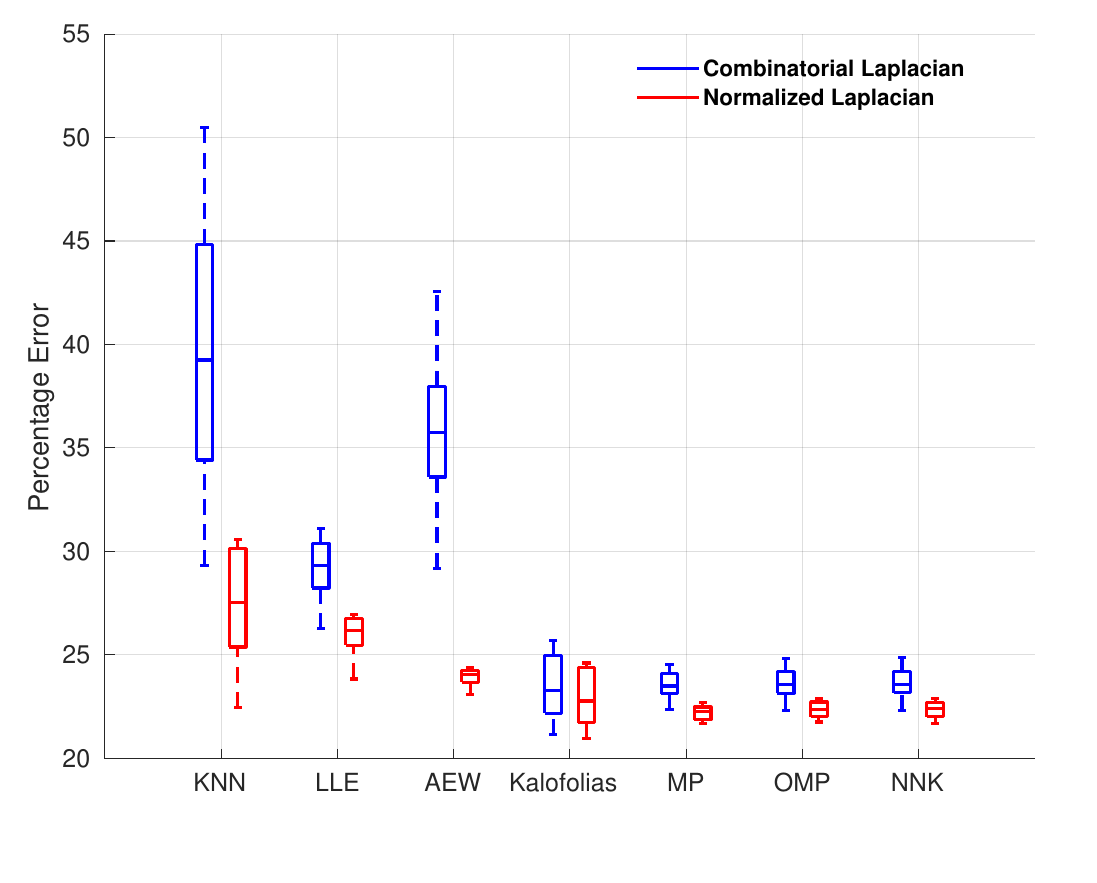}
\end{subfigure}
\caption{SSL performance (mean over 10 different subsets with 10 random initialization of available labels) with graphs constructed by different algorithms. \textbf{Left column:} Classification error at different percentage of labeled data on USPS (\textbf{top row}) and MNIST (\textbf{bottom row}) dataset with $\rk=30$. The time taken and the sparsity of the constructed graphs is indicated as part of the legend. \textbf{Right column:} Boxplots showing the robustness of graph constructions using $\mL$ and $\gL$ Laplacians for different choices of $\rk$ ($10,15, \dots, 50$) in SSL task with $10\%$ labeled data on USPS and MNIST.}
\label{fig:ssl_performance}
\end{figure}
\subsubsection{Laplacian Eigenmaps}
\label{sec:lap_eig}
\begin{figure*}[htbp]
\centering
\begin{subfigure}{0.48\textwidth}
\includegraphics[width=\textwidth]{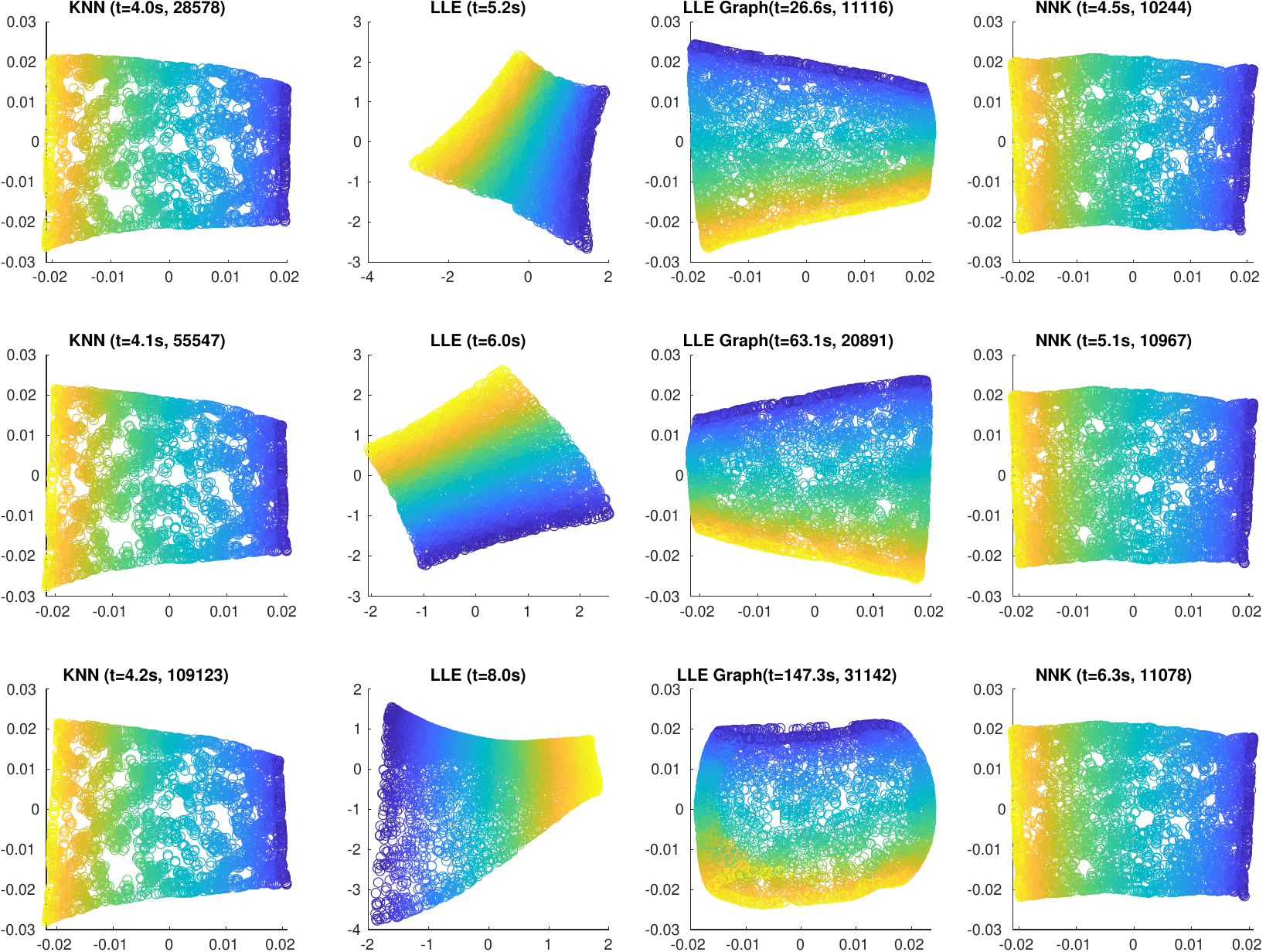}
\end{subfigure}
~
\begin{subfigure}{0.48\textwidth}
    \includegraphics[width=\textwidth]{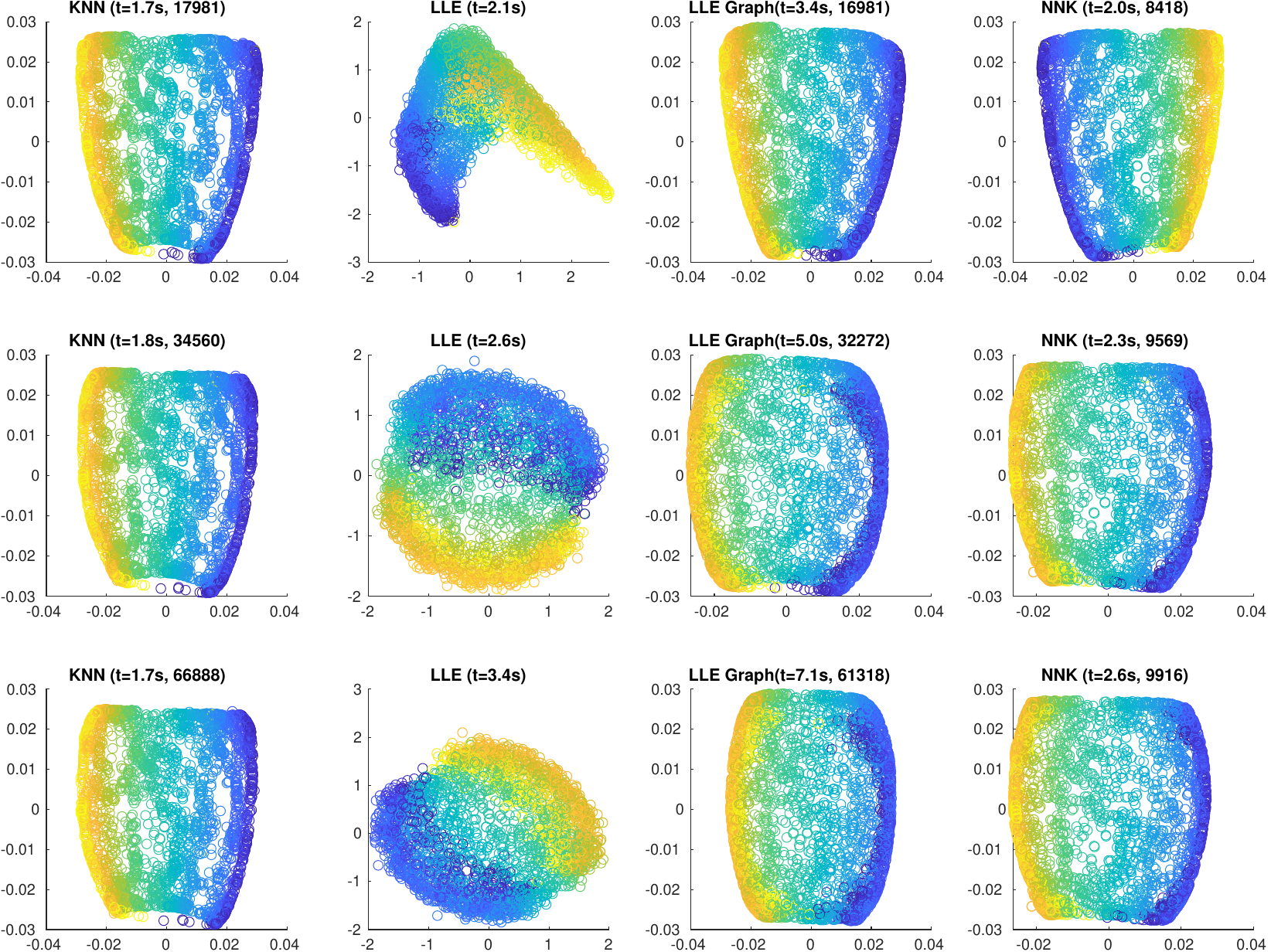}
\end{subfigure}
\caption{Two-dimensional embedding learned by different algorithms. \textbf{Left:}~ Swiss roll, \textbf{Right:}~Severed sphere. Each row corresponds to different choice of $\rk=10, 20, 40$~(\textbf{top to bottom}). We include the time taken to \emph{construct} the graph and the number of edges obtained for each method in the subplot title. We see that NNK graphs produce a robust embedding for both datasets with the sparsity of the graph constructed correlating with the intrinsic dimension of the dataset.}
\label{fig:embedding_robustness}
\end{figure*}
We demonstrate the effectiveness of the NNK graph construction in manifold learning using Laplacian Eigenmaps \cite{Belkin2001} with two standard synthetic datasets, namely, Swiss roll ($N=5000,d=3$) and Severed sphere ($N=3000, d=3$).
The embedding obtained using NNK is compared to that of $\rk$NN, LLE \cite{Wang2008} graphs, and that of the original LLE algorithm \cite{Roweis2323}.
\Figref{fig:embedding_robustness} presents a visual comparison of the embedding obtained using different methods and their robustness to the choice of $\rk$. We see that NNK graphs have better sparsity, runtime, and robustness in the embeddings obtained relative to the baselines. Moreover, the eigenmaps computation using NNK graphs is significantly faster due to the sparsity of the graph constructed. 
We note that the sparsity in NNK graphs captures the intrinsic dimensionality of the data, where the ratio of the number of edges to the number of data points corresponds to the dimension of the manifold associated with the datasets, namely, $2$ for Swiss roll corresponding to a plane and $3$ for the Severed sphere indicating that the data is close to a $3$-dimensional surface.  
\section{Conclusion}
\label{sec:conclusion}
We introduced a novel view of neighborhoods where we show that neighborhoods are non-negative sparse signal approximations. We then use this perspective to propose an improved approach, Non-Negative Kernel Regression (NNK), for neighborhood and graph constructions. We also present $l_1$-regularized and greedy pursuit approaches to  neighborhood definition while establishing its connection to our proposed NNK. We characterize the geometry of proposed neighborhood construction via the kernel ratio interval theorem and establish its implication in non-distance based kernels and other non-negative sparse approximation problems. 
Our experiments demonstrate that NNK performs better than earlier methods in neighborhood classification and graph-based semi-supervised and manifold learning. Further, we show that our approach leads to the sparsest solutions and is more robust to hyperparameters involved in both neighborhood and graph constructions. NNK has a desirable runtime complexity and can leverage neighborhood search tools developed for $\rk$NN. Moreover, the local nature of the NNK algorithm allows for further runtime improvements via parallelization and additional memory.
Future work will focus on investigating the relationship between kernel parameters and geometry,  addressing the impact of noise in neighborhood definitions, and developing multi-scale neighborhood representations.
\bibliographystyle{IEEEbib}
\bibliography{main}
\begin{IEEEbiographynophoto}{Sarath Shekkizhar}
received his bachelor (Electronics and Communication) and double master (Electrical Engineering, Computer Science) degrees from National Institute of Technology, Tiruchirappalli, India and University of Southern California (USC), Los Angeles, USA, respectively. He is currently working on his doctoral degree in Electrical and Computer Engineering at USC and is the recipient of the IEEE best student paper award at ICIP 2020. His research interests include graph signal processing, non-parametric methods, and machine learning. 
\end{IEEEbiographynophoto}
\begin{IEEEbiographynophoto}{Antonio Ortega}
received his undergraduate and doctoral degrees from the Universidad Politécnica de Madrid, Madrid, Spain and Columbia University, New York, USA, respectively.  In 1994 he joined the Electrical Engineering department at the University of Southern California, where he is currently Dean’s Professor. He was the Editor-in-Chief of the IEEE Transactions of Signal and Information Processing over Networks and recently served as a member of the Board of Governors of the IEEE Signal Processing Society. His recent research work is focused on graph signal processing, machine learning, multimedia compression and wireless sensor networks. Over $40$ PhD students have completed their PhD under his supervision and his work has led to over $400$ publications in international conferences and journals, as well as several patents. He is the author of the book \textit{Introduction to Graph Signal Processing} published by Cambridge University Press in 2022.
\end{IEEEbiographynophoto}
\clearpage
\appendices
\section{General Kernel Ratio Interval}
\label{subsec:general_KRI}
\begin{theorem}
\label{th:general_KRI}
Given a three node scenario
as in Fig.\ref{fig:3node} and a Mercer kernel, the necessary
and sufficient condition for two data points $i$ and $j$ to be connected
to $q$ in an NNK framework is
\begin{align}
\frac{\mK_{i,j}}{\mK_{j,j}} < \frac{\mK_{q,i}}{\mK_{q,j}} < \frac{\mK_{i,i}}{\mK_{i,j}} \label{eq:general_KRI}
\end{align}
\end{theorem}
\begin{remark}
Note that the general KRI condition in \autoref{th:general_KRI} is independent of the self-similarity or norm of the data to be represented, i.e., $\mK_{q,q}$. This makes intuitive sense from a neighborhood definition or signal approximation point of view: the  self-similarity of $q$ does not change the projection order of neighbors (nearest to farthest) since $\mK_{q,q}$ is non-negative for Mercer kernels. 
\begin{align*}
    \frac{\mK_{q,i}}{\mK_{i,i}} \leq \frac{\mK_{q,j}}{\mK_{j,j}} 
    \iff \frac{\mK_{q,i}}{\mK_{q,q}\mK_{i,i}} \leq \frac{\mK_{q,j}}{\mK_{q,q}\mK_{j,j}}
\end{align*}
\end{remark}
\begin{remark}
It should be also noted that if the data to be selected as neighbors (e.g., $i, j$) have unit self-similarity, then the condition of \autoref{th:general_KRI} reduces to the simpler KRI condition of \autoref{th:KRI} irrespective of the self-similarity  of the query point being approximated~i.e., $\mK_{q,q}$. 
\end{remark}

\section{Proofs}
\label{sec:appendix_proofs}
\subsection{Karush-Kuhn-Tucker optimality conditions of NNK}
\label{subsec:KKT}
A solution to (\ref{eq:neighborhood_nnk_objective}) satisfies the Karush-Kuhn-Tucker~(KKT) conditions as listed below.
\begin{subequations}
\label{eqs:KKT_condt}
\begin{align}
    \mK_{S,S}\vtheta - \mK_{S,q} -\vlambda = \vzero \label{eq:stationarity}\\
    \vlambda^\top\vtheta = 0 \label{eq:slackness}\\
    \vtheta \geq \vzero, \quad \vlambda \geq \vzero \label{eq:primary_dual_optimality}
\end{align}
\end{subequations}
We will use these conditions to analyze and prove the properties of the NNK framework in the following sections.
\subsection{Proof of Proposition \ref{prop:lle_generalization}}
\begin{proof}
An optimum solution at node $q$ satisfies the KKT conditions, specifically the stationarity condition~(\ref{eq:stationarity}) 
\begin{align}
\mK_{S,S}\vtheta = \mK_{S,q} \quad s.t \quad \vtheta_S \geq \mathbf{0} \label{eq:nnk_condition}
\end{align}
This can be interpreted as a solution to a set of linear equations under constraint. 
\begin{align}
    \sum_{j \in S}\theta_j\; \kappa(\vx_i, \vx_j) = \kappa(\vx_q, \vx_i) \quad \forall i \in S,\; \vtheta \geq \bf0 \label{eq:nnk_linear_eq}
\end{align}
Local Linear Embedding~(LLE) with positive weight constraint minimizes 
$$\mZ^\top\mZ \vw = 1, \text{ where } \mZ = \mX_S - \vx_q \quad s.t \; \sum_{i \in S}w_i = 1,\; \vw \geq \bf0$$
Let $\mM = \mZ^\top\mZ$, i.e.,  
\begin{align*}
    \mM_{i,j} = 
(\vx_i - \vx_q)^\top(\vx_j - \vx_q),
\quad \forall \; i,j  \in S.
\end{align*}
Then, LLE objective corresponds to solving a set of equations with constraints, namely
\begin{align}
    \sum_{j \in S}w_j\; (\vx_i - \vx_q)^\top(\vx_j - \vx_q) = 1 \quad \forall i \in S \nonumber\\
    \text{such that }\; \sum_{i \in S}w_i = 1,\quad \vw \geq \bf0 \label{eq:lle_linear_eq}
\end{align}
An equivalent set of equations is given by,
\begin{align*}
\sum_{j \in S}\mu_j \|\vx_ j- \vx_q\|\|\vx_i - \vx_q\| \, \frac{(\vx_j - \vx_q)^\top(\vx_i - \vx_q)}{2\|\vx_j - \vx_q\|\|\vx_i - \vx_q\|} = \frac{1}{2}.
\end{align*}
Since each weight $w_j$ is positive it can be factored as a product of positive terms, specifically let $$w_j = \mu_j \, \|\vx_j - \vx_q\|\;\|\vx_i - \vx_q\|$$
Further, using the constraint that weights add up to one, \eqref{eq:lle_linear_eq} is rewritten as
\begin{align}
\sum_{j \in S}w_j \, \frac{(\vx_j - \vx_q)^\top(\vx_i - \vx_q)}{2\|\vx_j - \vx_q\|\|\vx_i - \vx_q\|} &= 1 - \sum_{j \in S}\frac{w_j}{2}   \nonumber \\
\sum_{j \in S} w_j \left( \frac{1}{2} + \frac{(\vx_j - \vx_q)^\top(\vx_i - \vx_q)}{2\|\vx_j - \vx_q\|\|\vx_i - \vx_q\|}\right) &= 1 \nonumber \\
 \forall i \in S,& \;\vw \geq \bf0  \label{eq:lle_kernel_linear}
\end{align}
Combined with lemma \ref{lemma:cosine_lle_limit}, we see that \eqref{eq:lle_kernel_linear} is exactly the NNK objective with the kernel at node $q$ defined as in \eqref{eq:lle_cosine_kernel}, i.e.,
\begin{align}
    \sum_{j \in S}w_j\; \kappa_q(\vx_i, \vx_j) = \kappa_q(\vx_q, \vx_i) \quad \forall i \in S,\; \vw \geq \bf0 \label{eq:nnk_linear_eq_lle_kernel}
\end{align}
\end{proof}
\begin{lemma}
\label{lemma:cosine_lle_limit}
$\underset{\vh \to 0}{\lim} \; \kappa_q(\vx_q + \vh, \vx_i) = 1$
\end{lemma}
\begin{proof}
\begin{align*}
    \kappa_q(\vx_q, \vx_i) = \frac{1}{2} + \frac{(\vx_q - \vx_q)^\top(\vx_i - \vx_q)}{2\|\vx_q - \vx_q\|\|\vx_i - \vx_q\|}
\end{align*}
The above expression is indeterminate in its current form. Let's consider the limit of the expression.
\begin{align*}
    \underset{\vh \to 0}{\lim} \; \kappa_q(\vx_q + \vh, \vx_i) &= \underset{\vh \to 0}{\lim}\; \frac{1}{2} + \frac{(\vx_q + \vh - \vx_q)^\top(\vx_i - \vx_q)}{2\|\vx_q + \vh - \vx_q\|\|\vx_i - \vx_q\|} \\
    &= \frac{1}{2} + \underset{\vh \to 0,\, \alpha \to 0}{\lim}\; \frac{\|\vx_i - \vx_q\|\|\vh\|\cos\alpha}{2\|\vx_i - \vx_q\|\|\vh\|}
    \end{align*}
where $\alpha$ is the angle between $(\vx_i - \vx_q)$ and $\vh$.
\begin{align*}
\implies \underset{\vh \to 0}{\lim} \; \kappa_q(\vx_q + \vh, \vx_i)  &= 1 \quad  \text{since} \quad\underset{\alpha \to 0}{\lim}\cos\alpha = 1
\end{align*}
\end{proof}

\subsection{Proof of Proposition \ref{prop:l1_simplification}}
\begin{proof}
The Karush-Kuhn-Tucker (KKT) optimality conditions for \eqref{eq:nnk_objective_l1} are
\begin{subequations}
\begin{align}
\mK_{\gD,\gD}\vtheta^* - \mK_{\gD,q} + \eta\vone -\vlambda = \vzero \label{eq:l1_stationarity}\\
\lambda_j\theta^*_j = 0 \label{eq:l1_slackness}\\
\vtheta^* \geq \vzero, \quad \vlambda \geq \vzero 
\end{align}
\end{subequations}
The non negativity of both $\mK_{\gD, \gD}$ (kernels are non-negative by design) and $\vtheta^*$ (by constraint) ensures $\mK_{\gD,\gD}\vtheta^* \geq 0$.
This implies that the stationarity condition (\ref{eq:l1_stationarity}) is satisfied only when $ \eta\mathbf{1} - \mK_{\gD,\gD} -\vlambda \leq \mathbf{0}$. Combining with the slackness condition (\ref{eq:l1_slackness}) we have the results of the proposition.
\begin{align*}
    \eta - \mK_{q,j} > 0 \implies& \lambda_{j} > 0 \quad \because \eta - \mK_{q,j} -\lambda_j \leq 0 \\
    \implies& \theta^*_j = 0 \quad \because \lambda_j\theta^*_j=0 \qedhere
\end{align*}
\end{proof}

\subsection{Proof of Proposition \ref{prop:active_set}}
\begin{proof}
Under the partition, the objective of NNK (\ref{eq:neighborhood_nnk_objective}) can be rewritten using block partitioned matrices
$$ \mK_{S,S} = \begin{bmatrix}
\mK_{P, P} & \mK_{P, \bar{P}}\\
\mK^\top_{P, \bar{P}} & \mK_{\bar{P}, \bar{P}}
\end{bmatrix} , \quad
\mK_{S,q} = \begin{bmatrix}
\mK_{P, q}\\
\mK_{\bar{P},q}
\end{bmatrix}.$$
The optimization for the sub problem corresponding to indices in $\vtheta_P$ is
\begin{align}
\underset{\vtheta_P: \vtheta_P \geq 0}{\min} \;\;
 \frac{1}{2}\vtheta^\top_P\mK_{P, P}\vtheta_P - \mK^\top_{P,q}\vtheta_P.  \label{eq:sub_simplified_obj}
\end{align}

An optimal solution to this sub objective ($\vtheta^*_P$) satisfies the KKT conditions (\ref{eqs:KKT_condt}) namely,
\begin{align*}
\mK_{P, P}\vtheta^*_P - \mK_{P, q} -\vlambda_P = \bf0 \\
\vlambda_{P}^\top\vtheta^*_{P} = 0 \\
\vtheta^*_P \geq \bf0 , \;
\vlambda_P \geq \bf0 
\end{align*}
Specifically, given $\vtheta^*_P > \bf0$, the solution satisfies the stationarity condition (eq. \ref{eq:stationarity}), 
\begin{align}
    \mK_{P, P}\vtheta^*_P - \mK_{P, q} = \vzero \qquad \because \vlambda_P = \vzero \label{eq:inactive_constraint}
\end{align}
The zero augmented solution with $\vtheta_{\bar{P}} = \bf0$ is optimal for the original problem provided
\begin{align*}
\begin{bmatrix}
\mK_{P, P} & \mK_{P, \bar{P}}\\
\mK^\top_{P, \bar{P}} & \mK_{\bar{P}, \bar{P}}
\end{bmatrix}\begin{bmatrix}
\vtheta^*_P\\
\bf0
\end{bmatrix} - \begin{bmatrix}
\mK_{P, q}\\
\mK_{\bar{P},q}
\end{bmatrix} - \begin{bmatrix}
\vlambda_P\\
\vlambda_{\bar{P}}
\end{bmatrix} = \bf0 \\
\begin{bmatrix}
\vlambda_P &
\vlambda_{\bar{P}}
\end{bmatrix}\begin{bmatrix}
\vtheta^*_P\\
\bf0
\end{bmatrix} = 0 \\
\begin{bmatrix}
\vtheta^*_P\\
\bf0
\end{bmatrix} \geq \bf0 ,\;
\begin{bmatrix}
\vlambda_P\\
\vlambda_{\bar{P}}
\end{bmatrix} \geq \bf0 
\end{align*}
Given the optimality conditions on $\vtheta^*_P$ , the conditions for $[\vtheta^*_P \; \vtheta_{\bar{P}}]^\top$ to be optimal requires 
\begin{align}
\mK^\top_{P, \bar{P}}\vtheta^*_P - \mK_{\bar{P},q} -\vlambda_{\bar{P}} = \bf0 \nonumber \quad
\vlambda_{\bar{P}} \geq \bf0 \nonumber \\
\implies \mK^\top_{P, \bar{P}}\vtheta^*_P - \mK_{\bar{P},q} \geq \bf0 \label{eq:active_constraint}
\end{align}
\end{proof}
\subsection{Proof of \autoref{th:KRI}}
\subsubsection{Main proof}
\begin{proof}
An exact solution to objective \plainref{eq:neighborhood_nnk_objective} without the constraint on $\vtheta$ at a data point $q$ and set $S=\{i,j\}$ satisfies
\begin{align}
\begin{bmatrix}
1 & \mK_{i,j} \\ \mK_{i,j} & 1
\end{bmatrix}
\begin{bmatrix}
\theta_i \\ \theta_j
\end{bmatrix} = 
\begin{bmatrix}
\mK_{q,i} \\ \mK_{q,j}
\end{bmatrix}\\
\iff \theta_i + \theta_j\mK_{i,j} = \mK_{q,i} \nonumber\\
\theta_i\mK_{i,j} + \theta_j = \mK_{q,j} \nonumber
\end{align}
Taking the ratio of the equations
\begin{align}
    \frac{\theta_i + \theta_j\mK_{i,j}}{\theta_i\mK_{i,j} + \theta_j} = \frac{\mK_{q,i}}{\mK_{q,j}} \label{eq:obj_ratio} 
\end{align}
Now, $0\leq \mK_{i,j} \leq 1$, where $\mK_{i,j} = 1$ if and only if the points are same. Without loss of generality, let us assume points $i$ and $j$ are distinct. Then,
\begin{align}
\mK_{i,j} < 1 \leq \frac{\mK_{q,i}}{\mK_{q,j}} &\iff \mK_{i,j} < \frac{\theta_i + \theta_j\mK_{i,j}}{\theta_i\mK_{i,j} + \theta_j} \nonumber \\
\iff \theta_i + \theta_j\mK_{i,j} & > \theta_i\mK^2_{i,j} + \theta_j\mK_{i,j} \nonumber \\
\iff \theta_i > \theta_i\mK^2_{i,j} &\iff \theta_i > 0 \nonumber \\
\therefore\quad \theta_i > 0  &\iff \mK_{i,j} < \frac{\mK_{q,i}}{\mK_{q,j}} \label{eq:theta1_condt}
\end{align}
Similarly,
\begin{align}
\frac{\mK_{q,i}}{\mK_{q,j}} < \frac{1}{\mK_{i,j}} &\iff \frac{\theta_i + \theta_j\mK_{i,j}}{\theta_i\mK_{i,j} + \theta_j} <  \frac{1}{\mK_{i,j}} \nonumber \\
\iff \theta_i\mK_{i,j} + \theta_j &> \theta_i\mK_{i,j} + \theta_j\mK^2_{i,j} \nonumber \\
\iff \theta_j > \theta_j\mK^2_{i,j} &\iff \theta_j > 0 \nonumber \\
\therefore 
\quad\theta_j > 0 &\iff \frac{\mK_{q,i}}{\mK_{q,j}} < \frac{1}{\mK_{i,j}} \label{eq:theta2_condt}
\end{align}
Equations (\ref{eq:theta1_condt}) and (\ref{eq:theta2_condt}) give the necessary and sufficient condition of (\ref{eq:kernel_ratio_interval}). 
\begin{align*}
\mK_{i,j} < \frac{\mK_{q,i}}{\mK_{q,j}} < \frac{1}{\mK_{i,j}}.
\end{align*}
\end{proof} 

\subsubsection{Proof of Corollary \ref{corollary:plane_property}}
\begin{proof}
Let $j$ be a node beyond the plane as in \Figref{fig:plane}. For a Gaussian kernel with bandwidth $\sigma$, $\mK_{q,i} > \mK_{q,j}$. Thus, condition corresponding to (\ref{eq:theta1_condt}) is satisfied, $$\frac{\mK_{q,i}}{\mK_{q,j}} > 1 > \mK_{i,j} \implies \theta_i > 0$$
Let $\|\vx_q - \vx_i\| = a,\; \|\vx_q - \vx_j\|=b \;$ and $\alpha$ be the angle between the difference vectors. Then,
\begin{align*}
\frac{\mK_{q,i}}{\mK_{q,j}} &= \exp\left(\frac{1}{2\sigma^2} (b^2 - a^2) \right)\\
\mK_{i,j} &=  \exp\left(-\frac{1}{2\sigma^2} ((b\cos\alpha-a)^2 + (b\sin\alpha)^2) \right) \\
\mK_{i,j} &=  \exp\left(-\frac{1}{2\sigma^2} (b^2 + a^2 - 2ab\cos\alpha) \right)  \\
\frac{1}{\mK_{i,j}}&=\exp\left(-\frac{1}{2\sigma^2} (b^2 + a^2 - 2ab\cos\alpha) \right) \\
\end{align*}
Now, the condition on $\theta_j$, the inequality (\ref{eq:theta2_condt}) is reduced as
\begin{align*}
    \frac{\mK_{q,i}}{\mK_{q,j}} & < \frac{1}{\mK_{i,j}} \\
    \exp\left(\frac{1}{2\sigma^2} (b^2 - a^2) \right)  &< \exp\left(\frac{1}{2\sigma^2} (b^2 + a^2 - 2ab\cos\alpha)\right)   \\
    \exp\left(-\frac{a^2}{2\sigma^2} \right) & < \exp\left(\frac{1}{2\sigma^2} (a^2 - 2ab\cos\alpha) \right) \\
    \exp\left(\frac{1}{2\sigma^2} ab\cos\alpha \right) & <  \exp\left(\frac{a^2}{2\sigma^2} \right)   \\
    b\cos\alpha < a &\iff \theta_j  = 0 \quad \forall \; b\cos\alpha > a
\end{align*}
The condition $ b\cos\alpha > a$ corresponds to a hyperplane corresponding to points that belong to the half space not containing $q$.  
\end{proof}

\subsubsection{Proof of Corollary \ref{corollary:polytope_optimality}}
\begin{proof}
Statements A and B are related in that one is the contrapositive of the other. Here we prove statement A.\\
Equation (\ref{eq:inactive_set_eq}) can be interpreted as the boundaries of a convex polytope ($\gP$) formed by the half planes,
\begin{align}
    boundary(\gP) = \{\mK^\top_{P,i}\vtheta_P = \mK_{q,i} \quad i \in P\} \label{eq:half_plane_boundary_equations}
\end{align}
The interior of the polytope formed by the half planes is 
\begin{align}
    interior(\gP) = \{p\;:\; \mK^\top_{P,p}\vtheta_P - \mK_{q,p} < 0\}. \label{eq:interior_polytope}
\end{align}
Thus, 
\begin{align}
    \mK^\top_{P, j}\vtheta_P - \mK_{q,i} \geq 0 \label{eq:active_set_eq_k}
\end{align}
corresponds to a point outside  $\gP$. 
Now, to prove
\begin{align*}
    \exists  i  \in  P \; \colon \; \frac{\mK_{q,i}}{\mK_{q,j}} \geq \frac{1}{\mK_{i,j}}
\end{align*}
Let's assume there exists no such $i$. Then
\begin{align}
    \forall  i  \in  P \; \colon \; \frac{\mK_{q,i}}{\mK_{q,j}} &< \frac{1}{\mK_{i,j}}  \nonumber \\
    \mK_{q,i}\mK_{i,j}  &< \mK_{q,j} \nonumber\\
    \left(\mK^\top_{P,i}\vtheta_P \right)\mK_{i,j}  &< \mK_{q,j} \quad \quad \because i \in \gP \nonumber\\
    \left(\mK_{P,i}\mK_{i,j}\right)^\top\vtheta_P  &< \mK_{q,j}   
\end{align}
Using the triangle inequality corresponding to kernels for each term in $P$, we have
\begin{align}
    \mK^\top_{P, j}\vtheta_P &< \mK_{q,j} \qquad \text{since } \mK_{P, q} \leq \mK_{P,i}\mK_{i,j}  \label{eq:corollary_contradiction_eq}
\end{align}
\Eqref{eq:corollary_contradiction_eq} contradicts our earlier statement (eq. \ref{eq:active_set_eq_k}) on point $j$, namely, $j$ lies outside $\gP$ and hence does not belong to the interior as defined in \eqref{eq:interior_polytope}.
Thus, 
\begin{align*}
    \exists \; i \; \in \; P \; \colon \; \frac{\mK_{q,i}}{\mK_{q,j}} \geq \frac{1}{\mK_{i,j}}
\end{align*}
\end{proof}
\subsection{Proof of Proposition \ref{prop:RKHS_distance}}
\begin{proof}
\begin{align*}
    \Tilde{d}^2(i, j) & = (\vphi_i - \vphi_j)^t(\vphi_i - \vphi_j) \\
    &= \vphi_i^t\vphi_i - \vphi_i^t\vphi_j - \vphi_j^t\vphi_i + \vphi_j^t\vphi_j \\
    &= 2 - 2\vphi_i^t\vphi_j \; \because \|\vphi_i\| = \|\vphi_j\| = 1\\
    &= 2 - 2~\kappa(\vx_i, \vx_j)
\end{align*}
\end{proof}

\subsection{Proof of \autoref{thm: nnk_rkhs_geometry}}
\begin{proof}
The proof follows from the kernel ratio interval \autoref{th:KRI} and proposition \ref{prop:RKHS_distance}. 

Consider the condition for $\theta_j\;\neq\;0$:
\begin{align*}
    \frac{\mK_{q,i}}{\mK_{q,j}} < \frac{1}{\mK_{i,j}} \iff \mK_{q,i}\mK_{i,j} &< \mK_{q,j} \\
    \left[1 - \frac{\Tilde{d}^2(q,i)}{2}\right]\left[1 - \frac{\Tilde{d}^2(i,j)}{2}\right] &< \left[1 - \frac{\Tilde{d}^2(q,j)}{2}\right] \\
    1 - \frac{\Tilde{d}^2(q,i)}{2} - \frac{\Tilde{d}^2(i,j)}{2} + \frac{\Tilde{d}^2(q,i)\Tilde{d}^2(i,j)}{4} &< 1 - \frac{\Tilde{d}^2(q,j)}{2} \\
    \Tilde{d}^2(q,i) + \Tilde{d}^2(i,j) - \Tilde{d}^2(q,j) &> \frac{\Tilde{d}^2(q,i)\Tilde{d}^2(i,j)}{2}
\end{align*}
The proof for $\theta_i$ can be derived using a similar argument.
\end{proof}

\subsection{Proof of \autoref{th:general_KRI}}
\begin{proof}
A solution to the NNK objective without the constraint on $\vtheta$ at data point $q$ and set $S=\{i,j\}$ satisfies
\begin{align}
\begin{bmatrix}
\mK_{i,i} & \mK_{i,j} \\ \mK_{i,j} & \mK_{j,j}
\end{bmatrix}
\begin{bmatrix}
\theta_{i} \\ \theta_{j}
\end{bmatrix} = 
\begin{bmatrix}
\mK_{q,i} \\ \mK_{q,j}
\end{bmatrix}\\
\iff \theta_{i}\mK_{i,i} + \theta_{j}\mK_{i,j} = \mK_{q,i} \nonumber\\
\theta_{i}\mK_{i,j} + \theta_{j}\mK_{j,j} = \mK_{q,j} \nonumber
\end{align}
Taking the ratio of the above equations
\begin{align}
    \frac{\theta_{i}\mK_{i,i} + \theta_{j}\mK_{i,j}}{\theta_{i}\mK_{i,j} + \theta_{j}\mK_{j,j}} = \frac{\mK_{q,i}}{\mK_{q,j}} \label{eq:general_obj_ratio} 
\end{align}
In the following statements, we will show that under conditions of \eqref{eq:general_KRI} that the solutions will be positive and vice versa. Now, $\mK^2_{i,j} \leq \mK_{i,i}\mK_{j,j}$. This is because kernels with RKHS property are positive definite functions, with equality iff the points $i$ and $j$ are the same. Without loss of generality, let us assume points $i$ and $j$ are distinct. Then,
\begin{align}
\frac{\mK_{i,j}}{\mK_{j,j}} <  \frac{\mK_{q,i}}{\mK_{q,j}} \iff& \frac{\mK_{i,j}}{\mK_{j,j}} < \frac{\theta_{i}\mK_{i,i} + \theta_{j}\mK_{i,j}}{\theta_{i}\mK_{i,j} + \theta_{j}\mK_{j,j}} \nonumber \\
\iff \theta_{i}\mK_{i,i}\mK_{j,j} + \theta_{j}\mK_{i,j}\mK_{j,j} &> \theta_{i}\mK^2_{i,j} + \theta_{j}\mK_{j,j}\mK_{i,j} \nonumber \\
\iff \theta_{i}\mK_{i,i}\mK_{j,j} &> \theta_{i}\mK^2_{i,j} \nonumber\\
\iff \theta_{i} > 0 \qquad& \because \mK^2_{i,j} < \mK_{i,i}\mK_{j,j} \nonumber 
\end{align}
Thus,
\begin{align}
 \therefore \theta_{i} > 0 \iff \frac{\mK_{i,j}}{\mK_{j,j}} < \frac{\mK_{q,i}}{\mK_{q,j}} \label{eq:general_theta1_condt}
\end{align}
Similar result holds for $\theta_j$, namely,
\begin{align}
\theta_{j} > 0 \iff \frac{\mK_{q,i}}{\mK_{q,j}} < \frac{\mK_{i,i}}{\mK_{i,j}}. \label{eq:general_theta2_condt}
\end{align}
Equations (\ref{eq:general_theta1_condt}) and (\ref{eq:general_theta2_condt}) combined give the necessary and sufficient condition in the theorem and completes the proof. 
\begin{align*}
\frac{\mK_{i,j}}{\mK_{j,j}} < \frac{\mK_{q,i}}{\mK_{q,j}} < \frac{\mK_{i,i}}{\mK_{i,j}}
\end{align*}
\end{proof} 
\end{document}